\documentclass[11pt]{article}

\usepackage{epsfig,epsf,fancybox}
\usepackage{amsmath}
\usepackage{mathrsfs}
\usepackage{amssymb}
\usepackage{graphicx}
\usepackage{color}
\usepackage{multirow}
\usepackage{paralist}
\usepackage{verbatim}
\usepackage{galois}
\usepackage{algorithm}
\usepackage{algorithmic}
\usepackage{boxedminipage}
\usepackage{booktabs}
\usepackage{accents}
\usepackage{stmaryrd}
\usepackage{subcaption}
\usepackage{wrapfig}
\usepackage[bottom]{footmisc}
\usepackage[sort&compress]{natbib}
\usepackage{natbib}
\usepackage[normalem]{ulem}
\usepackage{xspace}

 \usepackage{tikz}
 \usepackage{wrapfig}
 \usepackage{color}
\usetikzlibrary{fadings}
\usetikzlibrary{patterns}
\usetikzlibrary{shadows.blur}
\usetikzlibrary{shapes}

\usepackage{array}		
\usepackage{booktabs}		
\usepackage[inline,shortlabels]{enumitem}		
\setenumerate{itemsep=0pt,parsep=\smallskipamount,topsep=\smallskipamount,left=\parindent}

\usepackage[colorlinks,linkcolor=magenta,citecolor=blue, pagebackref=true,backref=true]{hyperref}
\renewcommand*{\backrefalt}[4]{%
    \ifcase #1 \footnotesize{(Not cited.)}%
    \or        \footnotesize{(Cited on page~#2.)}%
    \else      \footnotesize{(Cited on pages~#2.)}%
    \fi}

\long\def\comment#1{}
\usepackage{nicefrac}

\newtheorem{theorem}{Theorem}[section]

\newtheorem{lemma}[theorem]{Lemma}

\newtheorem{assumption}[theorem]{Assumption}

\numberwithin{equation}{section}

\usepackage{mathtools}
\newcommand{\argmin}{\operatorname*{arg\,min}\limits}
\newcommand{\rr}{{\mathbb{R}}}
\newcommand{\calB}{{\mathcal{B}}}

\newcommand{\calL}{{\mathcal{L}}}
\newcommand{\EE}{{\mathbb{E}}}
\newcommand{\bb}[1]{\mathbf{#1}}
\newcommand{\bbw}[1]{\mathbf{w}_{#1}}

\newcommand{\bbg}{\mathbf{g}}
\newcommand{\bbH}{\mathbf{H}}
\newcommand{\bbQ}{\mathbf{Q}}
\newcommand{\bbJ}{\mathbf{J}}
\newcommand{\bbr}{\mathbf{r}}
\newcommand{\sigmag}{\sigma_g}

\newcommand{\ie}{{i.e.}}
\newcommand{\eg}{{e.g.}}
\newcommand{\norm}[1]{\left\lVert#1\right\rVert}

\newcommand{\sqbracket}[1]{\left[#1\right]}

\usepackage{accents}
\newcommand{\ubar}[1]{\underaccent{\bar}{#1}}

\begin{document}

\begin{center}


{\bf{\LARGE{
Exact Gauss-Newton Optimization \\ [.2cm] for Training Deep Neural Networks
}}}

\vspace*{.2in}
{\large{
\begin{tabular}{c}
Mikalai Korbit$^{\dagger}$ 
\and Adeyemi D. Adeoye$^{\dagger}$ 
\and Alberto Bemporad$^{\dagger}$
\and Mario Zanon$^{\dagger}$ \\
\end{tabular}
}}

\vspace*{.2in}

\begin{tabular}{c}
DYSCO (Dynamical Systems, Control, and Optimization)$^\dagger$ \\ 
IMT School for Advanced Studies Lucca, Italy \\
\end{tabular}

\vspace*{.2in}

\today

\vspace*{.2in}


\begin{abstract}
We present Exact Gauss-Newton (EGN), a stochastic second-order optimization algorithm that combines the generalized Gauss-Newton (GN) Hessian approximation with low-rank linear algebra to compute the descent direction. Leveraging the Duncan-Guttman matrix identity, the parameter update is obtained by factorizing a matrix which has the size of the mini-batch. This is particularly advantageous for large-scale machine learning problems where the dimension of the neural network parameter vector is several orders of magnitude larger than the batch size. Additionally, we show how improvements such as line search, adaptive regularization, and momentum can be seamlessly added to EGN to further accelerate the algorithm. Moreover, under mild assumptions, we prove that our algorithm converges in expectation to a stationary point of the objective. Finally, our numerical experiments demonstrate that EGN consistently exceeds, or at most matches the generalization performance of well-tuned SGD, Adam, GAF, SQN, and SGN optimizers across various supervised and reinforcement learning tasks.
\end{abstract}

\end{center}


\section{Introduction}

Optimization plays a pivotal role in Machine Learning (ML), 
with gradient-based methods being at the forefront. 
Stochastic Gradient Descent (SGD)~\cite{robbins1951stochastic}, a first-order 
stochastic optimization algorithm, 
and its accelerated versions such as 
momentum-based approaches \cite{nesterov1983method, sutskever2013importance, lucas2018aggregated, chen2022demon},
adaptive learning rates \cite{duchi2011adaptive, hinton2012neural, zeiler2012adadelta},
and a combination of the two \cite{kingma2014adam, dozat2016incorporating, zaheer2018adaptive},
have been instrumental in numerous ML applications.
For example, in Computer Vision (CV) ResNets~\cite{he2016deep} 
are trained with SGD,
AdaGrad~\cite{duchi2011adaptive} is used for training 
recommendation systems~\cite{naumov2019deep},
language models GPT-3~\cite{brown2020language} and 
LLaMA~\cite{touvron2023llama} are optimized 
with Adam~\cite{kingma2014adam} and AdamW~\cite{loshchilov2017decoupled}, 
respectively.
Despite their cheap and relatively easy-to-implement 
updates, first order-methods (FOMs) suffer from several  
shortcomings. FOMs are sensitive to hyper-parameter
selection, and the optimal hyper-parameter 
set typically does not transfer well across different problems
which leads to a costly procedure of hyper-parameter tuning.
Also, FOMs are slow to converge in the flat regions of the loss 
landscape, where the Hessian is ill-conditioned~\cite{sagun2017empirical}.

Second-order methods (SOMs) incorporate the (approximate)
curvature information into the update in order to  effectively
precondition the gradient vector.
In contrast to first-order algorithms, SOMs are shown 
to be robust to the selection of hyper-parameters~\cite{xu2020second}
and to potentially offer faster convergence~\cite{agarwal2017second, bollapragada2019exact}.
So far, the adoption of second-order methods for ML problems has been limited
due to the complexity of calculating and
storing the Hessian and the computational load of 
solving the linear system
$\mathbf{H}\mathbf{d}=-\mathbf{g}$, where $\mathbf{H}$ is 
the (approximate) Hessian matrix, 
$\mathbf{g}$ is the gradient of the 
loss function, and $\mathbf{d}$ is the descent direction.
Addressing these computational challenges,
most approaches use a combination of Hessian 
approximation and an efficient algorithmic technique 
for solving the linear system.
Common approximations to the Hessian include
diagonal scaling~\cite{yao2021adahessian, liu2023sophia},
the empirical Fisher matrix~\cite{ren2019efficient}, 
the quasi-Newton approach~\cite{schraudolph2007stochastic, byrd2016stochastic, berahas2016multi},
and the Gauss-Newton (GN) approximation~\cite{gargiani2020promise, tran2020stochastic, brust2021nonlinear}.

In this work, we follow the Gauss-Newton approach to Hessian approximation.
We apply a special derivation inspired by~\cite{adeoye2023score, adeoye2021sc} which uses
an efficient exact linear algebra 
identity---the Duncan-Guttman (DG) formula~\cite{duncan1944lxxviii,guttman1946enlargement}---to 
speed up the inversion of the Hessian matrix.
Compared to Hessian-free optimization 
(HFO)~\cite{martens2010deep,martens2011learning, kiros2013training} 
and Inexact Gauss-Newton (iGN)~\cite{tran2020stochastic}, this approach allows
Exact Gauss-Newton (EGN) to solve the system $\mathbf{H}\mathbf{d}=-\mathbf{g}$
\textit{exactly} with the same algorithmic complexity burden.
Moreover, compared to methods that directly apply the Sherman-Morrison-Woodbury (SMW) 
formula (see, e.g., \cite{ren2019efficient}), we solve for the descent direction in fewer matrix operations, 
thus reducing the algorithmic complexity.

Our contributions are as follows.
\begin{itemize}
    \item We propose the EGN algorithm, which 
    relies on a regularized Gauss-Newton Hessian matrix
    and
    exploits the Duncan-Guttman identity to 
    efficiently solve the linear system.

    \item We provide a theoretical analysis of the EGN algorithm 
    and establish that EGN finds a stationary point in expectation for a large enough iteration count.
    
    \item We evaluate the performance of EGN 
    on several supervised learning and
    reinforcement learning tasks 
    using various neural network architectures.    
\end{itemize}

\section{Related Work}\label{sec:related_work}

Our method can be viewed within the broader context of approximate second-order stochastic optimization.
Some notable approaches include 
diagonal scaling~\cite{yao2021adahessian, liu2023sophia},
Krylov subspace descent~\cite{vinyals2012krylov},
Hessian-free optimization~\cite{martens2010deep, martens2011learning, kiros2013training},
quasi-Newton approaches~\cite{schraudolph2007stochastic, byrd2016stochastic, berahas2016multi, wills2021stochastic, liu2022quasi}, 
Gauss-Newton ~\cite{botev2017practical, gargiani2020promise}
and Natural Gradient~\cite{amari1998natural, kunstner2019limitations} methods.     
A detailed overview of 
second-order optimization 
methods for large-scale machine learning problems 
can be found in~\cite{bottou2018optimization, sun2019survey, xu2020second}.

\begin{table*}[t]
    \centering
    \caption{A survey on Gauss-Newton methods for large-scale stochastic optimization}
    \label{tab:slm_methods_survey}
    \resizebox{\textwidth}{!}{%
        \begin{tabular}{|p{2cm}|p{4cm}|p{4cm}|p{4cm}|p{4cm}|}
            \hline
            \textbf{Algorithm} & \textbf{Jacobian Estimation} & \textbf{Solving the Linear System} & \textbf{Additional Improvements}  \\
            \hline
            SGN \cite{gargiani2020promise} & Exact via reverse mode autodiff & Approximate with CG & - \\
            \hline
            LM \cite{pooladzandi2022improving} & Exact via reverse mode autodiff  & Approximate with CG & Line search, momentum, uphill step acceptance   \\
            \hline
            SGN \cite{tran2020stochastic} & Exact via reverse mode autodiff  & Approximate with ADPG  & -  \\
            \hline
            SGN2 \cite{tran2020stochastic} & Approximate with SARAH estimators  & Approximate with ADPG & -   \\
            \hline
            NLLS1, NLLSL \cite{brust2021nonlinear} & Rank-1, Rank-L approximation & Exact with SMW formula & - \\
            \hline
            SMW-GN \cite{ren2019efficient} & Exact via reverse mode autodiff & Exact with SMW formula & Adaptive regularization  \\
            \hline
            EGN (this paper)
            & Any Jacobian estimation algorithm (exact via backpropagation is the default)
            & 
            Exact with DG identity (see Theorem \ref{thm:dg_identity} and Lemma \ref{lm:egn_direction})
            & Line search, adaptive regularization, momentum  \\
            \hline
        \end{tabular}
    }
\end{table*}

Most closely related to our work are the 
algorithms inspired by the 
Gauss-Newton approach.
Such methods approximate the Hessian of the loss function using only first-order sensitivities.
In practice, the damped version of the 
Gauss-Newton
direction is often calculated, 
forming the stochastic Levenberg-Marquardt (SLM)
group of algorithms.
We can classify SLM methods by three dimensions:
(a) by the type of the Jacobian 
estimation algorithm used;
(b) by the matrix inversion algorithm;
and (c) by additional adaptive parameters
and acceleration techniques.
Based on this paradigm
we summarize selected SLM algorithms
in Table~\ref{tab:slm_methods_survey}.

The Jacobian can either be calculated 
exactly through the reverse mode of 
automatic differentiation as proposed
by, e.g., 
\cite{gargiani2020promise, brust2021nonlinear, tran2020stochastic}
or be estimated approximately.
Low rank Jacobian estimation is suggested by
the NLLS1 and NLLSL algorithms~\cite{brust2021nonlinear} 
with experimental results 
showing almost on par performance with the 
exact Jacobian version of the methods.
SGN2~\cite{tran2020stochastic}
uses SARAH estimators for approximating
function values and Jacobians with 
SGN2 performing better than SGN~\cite{tran2020stochastic}
that assumes the exact Jacobian.
In~\cite{liu2023sophia} the 
Gauss-Newton-Bartlett (GNB) estimator is introduced
to adapt the GN method to large-scale classification problems.
EGN does not mandate a specific computation technique for 
the Jacobian. In our experiments we rely on backpropagation 
deferring other methods to further research.

Solving $\mathbf{H}\mathbf{d}=-\mathbf{g}$ naively  
for a neural network with $d$ parameters has complexity
$\mathcal{O}\left( d^3 \right)$.
The procedure becomes practically infeasible even for  
networks of moderate size, so several 
alternative approaches have been proposed.
Following~\cite{bollapragada2019exact}, we distinguish between \textit{inexact}
and \textit{exact} solutions to the linear system.
Inexact methods rely on iterative 
algorithms to solve the system approximately in as few 
iterations as possible.
Among such algorithms we mention
the Conjugate Gradient (CG) method
used in Hessian-free optimization~\cite{martens2011learning, kiros2013training}
as well as in GN methods like SGN~\cite{gargiani2020promise}
and LM~\cite{pooladzandi2022improving};
the Accelerated Dual Proximal-Gradient (ADPG) method
proposed in~\cite{tran2020stochastic}
for SGN and SGN2 solvers;
the Stochastic Gradient Iteration approach (SGI), 
analysed in~\cite{bollapragada2019exact}
as part of the Newton-SGI solver.
Exact methods are typically based on linear algebra 
identities.
For example, in \cite{ren2019efficient} and \cite{brust2021nonlinear} 
the system is solved exactly 
with the Sherman-Morrison-Woodbury (SMW) formula.
Contrarily, we follow ~\cite{adeoye2023score} and derive the EGN update formula using the Duncan-Guttman 
matrix identity~\cite{duncan1944lxxviii,guttman1946enlargement}.

State-of-the-art implementations of the GN algorithm 
often include additional improvements
to address the issues of stochasticity of the Hessian matrix,
computational load of calculating the Jacobian
and adaptive hyper-parameter tuning.
Just as with the gradient, the stochastic sampling 
introduces noise in the Hessian which leads 
to the erroneous descent direction.
A common solution to combat noisy estimates 
is to add temporal averaging (or momentum),
e.g., with exponential moving averages (EMA).
Examples of such approach are AdaHessian~\cite{yao2021adahessian} and
Sophia~\cite{liu2023sophia} that keep EMA of a diagonal Hessian,
as well as ~\cite{kiros2013training} that incorporates momentum 
into the HFO framework.
The idea of re-using the Hessian estimate from previous 
iterations is formalized in~\cite{doikov2023second} 
showing that evaluating the Hessian ``lazily'' once per $k$
iterations significantly reduces the computational burden 
while at the same time does not degrade the performance that much.
Although second-order methods typically 
require less tuning \cite{botev2017practical, gargiani2020promise, xu2020second},
the SLM approach still requires setting 
a learning rate $\alpha$ and
the regularization parameter $\lambda$.
The line search for $\alpha$, common in 
the deterministic optimization,
is problematic in the stochastic 
setting due to the high variance 
of the loss gradient norm  
\cite{curtis2020adaptive}. Still, there are promising attempts 
to incorporate line search into quasi-Newton 
methods~\cite{wills2021stochastic}, HFO~\cite{kiros2013training}
as well as Gauss-Newton~\cite{pooladzandi2022improving}.
Adaptive regularization 
in a manner similar to the 
deterministic Levenberg-Marquardt approach
is proposed 
in~\cite{hong2020stochastic, ren2019efficient}.

\section{Preliminaries}\label{sec:preliminaries}

We use boldface letters to denote vectors and matrices. The $n\times n$ identity matrix is denoted by $\bb{I}_n$, and we omit the subscript when the size is clear from the context. 
The subscript $t$ represents the iteration within
the optimization loop, 
and may be omitted to avoid overloading the notation. 
We define the standard inner product between two vectors $\bb{x}, \bb{y}$ as $\langle\bb{x},\bb{y}\rangle \coloneqq \bb{x}^\top \bb{y}$. 
The standard Euclidean norm is denoted by $\|\cdot\|$, 
and the expected value of a random variable is denoted by $\EE[\cdot]$.

We adopt the Empirical Risk Minimization (ERM)
framework~\cite{pml1Book} and 
consider the problem of finding the weights $\mathbf{w}\in \rr^{d}$ of a parametric function $\Phi:\rr^{m}\times\rr^{d}\rightarrow\rr^{c}$, 
e.g., a neural network, such that 
it minimizes the empirical risk over 
a dataset $\mathcal{D}$ consisting of $N$ pairs
$(\mathbf{y}_{i},\mathbf{x}_{i})$ where 
$\mathbf{x}_{i}\in\rr^{m}$
is a vector of features
and $\mathbf{y}_{i}\in\rr^c$, $c\ge 1$ is a target vector.
We want to solve the following optimization problem:
\begin{equation}\label{eq:problem}
\bb{w}^* \in \argmin_{\mathbf{w\in}\mathbb{R}^{d}}\mathcal{L}_{N}\left(\mathbf{w}\right) \coloneqq \EE_{\xi\sim P_\xi}[\calL_N(\bbw{};\xi)],
\end{equation}
where $\xi$ is a random variable with distribution $P_\xi$.
The objective function is expressed as a finite-sum of functions $\calL_N(\bb{w};\xi_i)$ over the realization $\mathcal{D}$ of $\xi$, and
\begin{align}\label{eq:risk}
\mathcal{L}_{N}\left(\mathbf{w}\right)\coloneqq\frac{1}{N}\sum_{i=1}^{N}\ell\left(\mathbf{y}_{i},\Phi(\mathbf{x}_{i};\mathbf{w})\right), 
\end{align}
is the empirical risk with $\ell\colon \rr^c\times \rr^c \to \rr$ -- a loss function,
involving a single pair 
$\left( \mathbf{y}_{i},\mathbf{x}_{i}\right)$ only.

We assume all these functions to be twice differentiable with respect to $\mathbf{w}$. Note that, while this assumption could be partially relaxed, 
we stick to it for the sake of simplicity.

\subsection{Gradient-based Optimization}

Problem~\eqref{eq:problem} is typically solved by 
variations of the Stochastic Gradient Descent (SGD) method
by sampling mini-batches $\mathcal{B}_t$  from $\mathcal{D}$
rather than processing the entire dataset 
in each iteration.
We denote the loss on a mini-batch $\mathcal{L}_b$ as
\begin{align}\label{eq:batch_loss}
\mathcal{L}_{b}\left(\mathbf{w}\right)\coloneqq\frac{1}{b}\sum_{i=1}^{b}\ell\left(\mathbf{y}_{i},\Phi(\mathbf{x}_{i};\mathbf{w})\right),
\end{align}
where $b$ is the batch size.
The iterations take the form
\begin{equation}\label{eq:gd_update}    \mathbf{w}_{t+1}\leftarrow\mathbf{w}_{t}+\alpha_{t}\mathbf{d}_{t},
\end{equation}
where $\alpha_{t}>0$ is a learning rate and $\mathbf{d}_{t}$ is 
a descent direction obtained from the gradient.
In general, we have $\mathbf{d}_{t}=-\mathbf{C}_{t}\mathbf{g}_{t}$ 
where $\mathbf{C}_{t}$ is a preconditioning matrix 
that scales, rotates, and shears 
the gradient of the mini-batch loss
$\mathbf{g}\coloneqq \nabla_{\mathbf{w}}\mathcal{L}_b$.
Notice that by setting $b=1$, 
$\mathbf{C}_{t}=\bb{I}_d$
and $\mathbf{d}_{t}=-\mathbf{g}_t$
we recover the incremental SGD update~\cite{robbins1951stochastic}. In practice, 
the preferred training algorithm is often
an accelerated version of mini-batch SGD.

Minimizing the quadratic approximation of the batch loss
leads to the following linear system
\begin{align}\label{eq:linear_system}
    \mathbf{H}_{t}\mathbf{d}_{t}=-\mathbf{g}_{t},
\end{align}
where $\mathbf{H}_t \coloneqq \nabla_{\mathbf{w}}^{2}\mathcal{L}_b$ 
is the Hessian matrix of the mini-batch loss.  
Finding the direction $\mathbf{d}_t$ by solving the system~\eqref{eq:linear_system}
using $\mathbf{H}_t$ or its approximation
defines the broad spectrum of second-order methods.
Setting $\mathbf{C}_{t}=\mathbf{H}^{-1}_{t}$ 
results in Newton's method.
This method suffers from several drawbacks: 
(a) one needs to compute second-order derivatives with respect to $\mathbf{w}$; 
(b) $\mathbf{C}_{t}$ has to be positive-definite to ensure descent; 
(c) the linear system~\eqref{eq:linear_system}
must be solved, which in general scales cubically with the dimension of $\mathbf{w}$;
and (d) since $\mathbf{H}_t$ is a noisy estimate of the true 
Hessian of the empirical risk $\mathcal{L}$, $\mathbf{H}^{-1}_t$ can result 
in suboptimal conditioning.
These issues and the fact that $\mathbf{w}$ is of rather large dimension 
have made the direct application of 
Newton's method practically irrelevant in ML applications. 
Indeed, for problems like Large Language Model (LLM) 
pre-training~\cite{vaswani2017attention, brown2020language, touvron2023llama}
or CV tasks~\cite{he2016deep}, 
$\mathbf{w}$ can be extremely high-dimensional,
e.g., 
$314 \cdot 10^9$ 
parameters for Grok-1~\cite{grok1} model, 
$8 \cdot 10^9$ to $405 \cdot 10^9$ parameters for LLaMA-family models~\cite{dubey2024llama}
and $0.27 \cdot 10^6$ to $19.4 \cdot 10^6$ parameters for ResNets~\cite{he2016deep},
which explains 
why accelerated first-order methods are usually preferred.

In order to make SOMs scalable for ML applications
one could instead approximate the inverse Hessian, i.e., $\mathbf{C}_{t} \approx \mathbf{H}^{-1}_{t}$. Examples of such preconditioning include 
diagonal scaling (e.g., with Hutchinson 
method~\cite{yao2021adahessian, liu2023sophia})
with the idea of extracting the (approximate)
diagonal elements of $\mathbf{H}_t$ while neglecting the 
off-diagonal terms; 
the quasi-Newton approach~\cite{byrd2016stochastic,berahas2016multi}
that approximates the Hessian using the information 
from past and current gradients;
and the Gauss-Newton method~\cite{gargiani2020promise, ren2019efficient} 
that leverages the Jacobian of the residuals, 
neglecting second-order cross-derivatives, particularly suitable 
for loss functions structured as~\eqref{eq:risk}. 
In this paper, we use the Gauss-Newton Hessian approximation
which is shown to provide a good approximation 
of the true Hessian for 
ML applications~\cite{sagun2017empirical, papyan2018full}.

\subsection{Generalized Gauss-Newton Hessian Approximation}

We consider the Generalized Gauss-Newton (GGN) 
Hessian approximation scheme~\cite{schraudolph2002fast, bottou2018optimization, papyan2018full}
which is suited 
for both regression and multi-class classification tasks.
The GGN Hessian approximation 
is constructed using only first-order 
sensitivities (see the derivation in Appendix~\ref{apx:derivations})
to obtain
\begin{align}\label{eq:H_gn}
    \mathbf{H}^{\mathrm{GN}}=\frac{1}{b}\mathbf{J}^{\top}\mathbf{Q}\mathbf{J}
\end{align}
where we vertically stack individual Jacobians 
$\mathbf{J}_{\Phi_i}:=\frac{\partial\Phi(\mathbf{x}_{i};\mathbf{w})}{\partial\mathbf{w}}$ 
for each sample in the batch $\mathcal{B}$ to form
$\mathbf{J}=\left[\begin{array}{ccc} \mathbf{J}_{\Phi_1} & \dots & \mathbf{J}_{\Phi_b}\end{array}\right]^{\top} \in \rr^{bc \times d}$
and construct a block diagonal matrix
$\mathbf{Q} = \text{blkdiag}(\mathbf{Q}_{\ell_1}, \mathbf{Q}_{\ell_2}, \ldots, \mathbf{Q}_{\ell_b}) \in \rr^{bc \times bc}$,
where $\mathbf{Q}_{\ell_i}=\frac{\partial^{2}\ell(\mathbf{y}_{i},\Phi(\mathbf{x}_{i};\mathbf{w}))}{\partial \Phi ^{2}}\in\rr^{c\times c}$. 
As pointed out in~\cite{sankar2021deeper}, approximation~\eqref{eq:H_gn}
is justified since
the true Hessian is dominated by the term
$\mathbf{J}^{\top}\mathbf{Q}\mathbf{J}$. 
Using the inverse of the GGN Hessian as a preconditioner 
yields the following direction
\begin{align}\label{eq:gn_batch_upd}
    \mathbf{d}_{t}^{\mathrm{GN}}=-\left(\frac{1}{b}\mathbf{J}_{t}^{\top}\mathbf{Q}_{t}\mathbf{J}_{t}\right)^{-1}\mathbf{g}_{t},
\end{align}
which, by defining $\Delta\mathbf{w}=\mathbf{w}-\mathbf{w}_{t}$, corresponds to the solution of the quadratic program
\begin{align}\label{eq:gen_naive_GN_step}
    \mathbf{d}_{t}^{\mathrm{GN}}=\argmin_{\Delta\mathbf{w}}\,\frac{1}{2}\Delta\mathbf{w}^{\top}\underbrace{\frac{1}{b}\mathbf{J}_{t}^{\top}\mathbf{Q}_{t}\mathbf{J}_{t}}_{\mathbf{H}_{t}^{\mathrm{GN}}}\Delta\mathbf{w}+\mathbf{g}_{t}^{\top}\Delta\mathbf{w}.
\end{align}

The Gauss-Newton step~\eqref{eq:gn_batch_upd} solves issues (a) and, 
partially, (b), since no second-order derivatives need to be 
computed and the Hessian approximation is positive semi-definite by construction provided that 
the loss function $\ell\left(\mathbf{y}_{i},\Phi(\mathbf{x}_{i};\mathbf{w})\right)$
is convex. A positive-definite Hessian approximation can easily be  obtained from $\mathbf{H}^\mathrm{GN}$, 
e.g., by adding to it a small constant times the identity matrix. 
This approach is called Levenberg-Marquardt (LM)~\cite{nocedal2006numerical} 
and is often used in practice,
such that
\begin{align}\label{eq:H_lm}
    \mathbf{H}^{\mathrm{LM}}=\frac{1}{b}\mathbf{J}^{\top}\mathbf{Q}\mathbf{J}+\lambda\bb{I}_{d}
\end{align}
with $\lambda>0$.
The regularizer $\lambda\bb{I}_{d}$ serves a dual purpose:
it ensures that the approximate Hessian matrix is invertible
and also it helps to avoid over-fitting~\cite{pml1Book}.

The Gauss-Newton update, however, still potentially suffers from issue (c), 
i.e., the need to solve a linear system, which can be of cubic complexity,
and (d) the noise in the Hessian estimate.
Since issue (c) is a centerpiece of our method, 
and issue (d) is an inherent part of stochastic sampling  
we defer the discussion
on both of them to Section~\ref{sec:algo}.

We conclude by examining the GGN update 
in two common machine learning tasks---regression 
and multi-class classification---and how it is derived in each case.

\paragraph{Regression} Regression is a predictive modeling task 
where the objective is to predict a scalar numeric target. 
For regression tasks we define the loss function $\ell$ 
as the mean squared error (MSE)
\begin{align}\label{eq:loss_mse}
\ell\left(\mathbf{y}_{i},\Phi(\mathbf{x}_{i};\mathbf{w})\right)\coloneqq \frac{1}{2}\left(\Phi(\mathbf{x}_{i};\mathbf{w})-\mathbf{y}_{i}\right)^{2}.
\end{align}

We can show that the 
gradient and the GN Hessian for the MSE loss read
\begin{align}
    \mathbf{g}=\frac{1}{b}\mathbf{J}^{\top}\mathbf{r}, \qquad \mathbf{H}^{\text{GN}}=\frac{1}{b}\mathbf{J}^{\top}\mathbf{J},
\end{align}
where $\mathbf{J} \in \rr^{b \times d}$ is a matrix of stacked Jacobians,
and $\mathbf{r} \in \rr^{b}$ is a vector of residuals
defined as $\mathbf{r}:=\left[\begin{array}{ccc}  \Phi(\mathbf{x}_{1};\mathbf{w}) - \mathbf{y}_1 & \dots &  \Phi(\mathbf{x}_{b};\mathbf{w}) - \mathbf{y}_b \end{array}\right]^{\top}$.

\paragraph{Multi-class Classification} The task of multi-class classification is to predict a 
correct class from $c$ classes given a vector of features $\mathbf{x}_i$.
For such problems the output of the neural network $\Phi$ is a
$c$-dimensional vector of prediction scores (logits) 
$\mathbf{z}_{i}=\Phi(\mathbf{x}_{i};\mathbf{w})$ 
and the target vector 
is a one-hot encoded vector
$\mathbf{y}_{i}=e_k$,
where $e_k$ denotes column $k$ of the identity matrix $\bb{I}_c$ 
and $k$ is the index of the correct class.
We define the loss function as a
softmax cross-entropy loss (CE)
\begin{align}\label{eq:loss_ce}
\ell\left(\mathbf{y}_{i},\Phi(\mathbf{x}_{i};\mathbf{w})\right)\coloneqq-\sum_{k=1}^{c}\mathbf{y}_{i,k}\log\left(\sigma\left(\mathbf{z}_{i,k}\right)\right),
\end{align}
where
$\sigma(\mathbf{z}_{i,k})=\frac{e^{\mathbf{z}_{i,k}}}{\sum_{j=1}^{c}e^{\mathbf{z}_{i,j}}}$,
$\mathbf{y}_{i,k}$ and $\mathbf{z}_{i,k}$ 
are $k$-th elements of vectors $\mathbf{y}_{i}$
and $\mathbf{z}_{i}$ respectively.

The gradient and the GN Hessian for the CE loss are
\begin{align}
    \mathbf{g}=\frac{1}{b}\mathbf{J}^{\top}\mathbf{r}, \qquad \mathbf{H}^{\text{GN}}=\frac{1}{b}\mathbf{J}^{\top}\mathbf{Q}\mathbf{J},
\end{align}
where $\mathbf{J} \in \rr^{bc \times d}$ 
is a matrix of stacked Jacobians,
$\mathbf{r} \in \rr^{bc}$ is a vector of 
(pseudo-)residuals defined as
$\mathbf{r}:=\left[\begin{array}{ccc}
\left(\sigma\left(\Phi(\mathbf{x}_{1};\mathbf{w})\right)-\mathbf{y}_{1}\right)^{\top} & \dots & \left(\sigma\left(\Phi(\mathbf{x}_{b};\mathbf{w})\right)-\mathbf{y}_{b}\right)^{\top}\end{array}\right]^{\top}$,
and $\mathbf{Q} \in \rr^{bc \times bc}$ is a 
block diagonal matrix of stacked matrices $\mathbf{Q}_{\ell_i}$
that each have 
$\sigma(\mathbf{z}_{i,k})(1-\sigma(\mathbf{z}_{i,k}))$ 
across the diagonal and
$-\sigma(\mathbf{z}_{i,k})\sigma(\mathbf{z}_{i,l})$ off-diagonal.

\section{Algorithm}\label{sec:algo}

We are now ready to present the EGN 
algorithm. First, we will discuss how one can efficiently solve the linear system. Then, we will discuss further enhancements to the basic algorithm.
Next, we address issue (c), i.e., the problem of
finding the solution of the symmetric linear system
$\mathbf{H}_{t}\mathbf{d}_{t}=-\mathbf{g}_{t}$.

Substituting the exact Hessian with the regularized Gauss-Newton 
Hessian yields
\begin{align}\label{eq:lm_linear_system}
    \left(\frac{1}{b}\mathbf{J}_{t}^{\top}\mathbf{Q}_{t}\mathbf{J}_{t}+\lambda_{t}\bb{I}_{d}\right)\mathbf{d}_{t}^{\mathrm{LM}}=-\frac{1}{b}\mathbf{J}_{t}^{\top}\mathbf{r}_{t},
\end{align}
where for the MSE loss $c=1$ and $\mathbf{Q}_{t}=\bb{I}_b$.
Solving~\eqref{eq:lm_linear_system} for 
$\mathbf{d}_{t}^{\mathrm{LM}}$
requires one to factorize matrix $\mathbf{H}^{\mathrm{LM}}$,
carrying a complexity of $\mathcal{O}\left( d^3 \right)$.
We notice, however, that in practice one often has $d \gg bc$, 
i.e., the parameter vector is of very high dimension, e.g., $d> 10^{6}$. 
In that case, the GN Hessian matrix~\eqref{eq:H_gn} is low-rank by construction.
That allows us to transfer the computationally expensive inversion 
operation from the high-dimensional $d\times d$ space (which is the original 
dimension of the Hessian) to the low-dimensional $bc \times bc$ space.

To that end, we follow an approach similar to the ones in~\cite{adeoye2023score,adeoye2021sc} which propose to utilize the Duncan-Guttman identity~\cite{duncan1944lxxviii, guttman1946enlargement}. We present this in Theorem~\ref{thm:dg_identity} and Lemma~\ref{lm:egn_direction} below for the Levenberg-Marquardt direction. For a general presentation which considers smooth regularization functions, we refer the interested reader to~\cite[Section 3]{adeoye2023score}.

\begin{theorem}[\cite{duncan1944lxxviii, guttman1946enlargement}]\label{thm:dg_identity}
Assuming $\mathbf{A}$ and $\mathbf{D}$ are full-rank matrices,
the following identity holds
\begin{align}\label{eq:dg_identity}
    \left(\mathbf{A}-\mathbf{B}^{\top}\mathbf{D}^{-1}\mathbf{C}\right)^{-1}\mathbf{B}^{\top}\mathbf{D}^{-1}  = \mathbf{A}^{-1}\mathbf{B}^{\top}\left(\mathbf{D}-\mathbf{C}\mathbf{A}^{-1}\mathbf{B}^{\top}\right)^{-1}.
\end{align}
\end{theorem}
By observing that Equation~\eqref{eq:lm_linear_system} defining
the direction $\mathbf{d}^{\mathrm{LM}}$
has the form of the left-hand side of
the identity~\eqref{eq:dg_identity}
we state the following.

\begin{lemma}\label{lm:egn_direction}
    The Levenberg-Marquardt direction $\mathbf{d}$ (Equation~\eqref{eq:lm_linear_system})
    for both MSE and CE loss functions
    can be computed using Algorithm~\ref{alg:slm_direction_finder}.
\end{lemma}
\begin{algorithm}[H]
    \caption{EGN direction function}
    \label{alg:slm_direction_finder}
    \begin{algorithmic}[1]
        \STATE {\bfseries Input:} (pseudo-)residuals $\mathbf{r}$,
        stacked Jacobians $\mathbf{J}$,
        regularizer $\lambda$, 
        batch size $b$.
        
        \STATE Solve the linear system for $\delta$: $\left(\mathbf{Q}\mathbf{J}\mathbf{J}^{\top}+b\lambda\bb{I}_{bc}\right) \delta=\mathbf{r}$ \hfill \textcolor{brown}{\; // $\mathcal{O}\left(b^2 c^2 d + b^3 c^3 \right)$}
        
        \STATE Calculate direction $\mathbf{d}^{\mathrm{LM}}=-\mathbf{J}^{\top} \delta$ \hfill \textcolor{brown}{// $\mathcal{O}\left( bcd \right)$}
        
        \STATE \textbf{Return} $\mathbf{d}^{\mathrm{LM}}$
    \end{algorithmic}
\end{algorithm}

\begin{proof}
Substituting
$\mathbf{A}=b\lambda\bb{I}_{d}$, 
$\mathbf{B}^{\top}=-\mathbf{J}^{\top}$,
$\mathbf{C}=\mathbf{J}$ and 
$\mathbf{D}=\mathbf{Q}^{-1}$
into Equation~\eqref{eq:dg_identity} we get
\begin{align}
    \left(b\lambda \bb{I}_{d} + \mathbf{J}^{\top} \mathbf{Q} \mathbf{J}\right)^{-1} \mathbf{J}^{\top} \mathbf{Q} = -\left(b\lambda \bb{I}_{d}\right)^{-1} \mathbf{J}^{\top} 
    \left(\mathbf{Q}^{-1} + \mathbf{J} \left(b\lambda \bb{I}_{d}\right)^{-1} \mathbf{J}^{\top}\right)^{-1}
\end{align}
Multiplying both sides by $\mathbf{Q}^{-1}\mathbf{r}$ yields
\begin{align}\label{eq:substituting_into_dg}
    \left(b\lambda \bb{I}_{d} + \mathbf{J}^{\top} \mathbf{Q} \mathbf{J}\right)^{-1} \mathbf{J}^{\top} \mathbf{r} = -\left(b\lambda \bb{I}_{d}\right)^{-1} \mathbf{J}^{\top} 
    \left(\mathbf{Q}^{-1} + \mathbf{J} \left(b\lambda \bb{I}_{d}\right)^{-1} \mathbf{J}^{\top}\right)^{-1} \mathbf{Q}^{-1} \mathbf{r}
\end{align}
where the lhs of~\eqref{eq:substituting_into_dg} 
is now equivalent to explicitly solving for $\mathbf{d}^{\mathrm{LM}}$ 
in Equation~\eqref{eq:lm_linear_system}. 
Simplifying the rhs of~\eqref{eq:substituting_into_dg} results in
\begin{align}
    \mathbf{d}^{\mathrm{LM}}=-\mathbf{J}^{\top}\left(b\lambda\mathbf{Q}^{-1}+\mathbf{J}\mathbf{J}^{\top}\right)^{-1}\mathbf{Q}^{-1}\mathbf{r}.
\end{align}
Applying the inverse of a product property,
$\mathbf{B}^{-1}\mathbf{A}^{-1}=\left(\mathbf{A}\mathbf{B}\right)^{-1}$,
to the expression 
$\left(b\lambda\mathbf{Q}^{-1}+\mathbf{J}\mathbf{J}^{\top}\right)^{-1}\mathbf{Q}^{-1}$
we have
\begin{align}
\left(b\lambda \mathbf{Q}^{-1} + \mathbf{J} \mathbf{J}^{\top}\right)^{-1} \mathbf{Q}^{-1} = \left(\mathbf{Q} \left(b\lambda \mathbf{Q}^{-1} + \mathbf{J} \mathbf{J}^{\top}\right)\right)^{-1} = \left(b\lambda \bb{I}_{bc} + \mathbf{Q} \mathbf{J} \mathbf{J}^{\top}\right)^{-1}.
\end{align}
So that the direction can be calculated as 
\begin{align}\label{eq:direction_lm}
\mathbf{d}^{\mathrm{LM}}=-\mathbf{J}^{\top}\left(\mathbf{Q}\mathbf{J}\mathbf{J}^{\top}+b\lambda\bb{I}_{bc}\right)^{-1}\mathbf{r},
\end{align}
which corresponds to the procedure outlined in Algorithm~\ref{alg:slm_direction_finder}.
\end{proof}

\subsection{Comparison to Existing Methods}

The key property of Algorithm~\ref{alg:slm_direction_finder}
is that the system~\eqref{eq:linear_system} is solved \textit{exactly}
in contrast to approximate (or \textit{inexact}) solutions
that underpin algorithms such as
HFO~\cite{martens2010deep, kiros2013training},
Newton-SGI~\cite{bollapragada2019exact},
LiSSA~\cite{agarwal2017second},
SGN~\cite{gargiani2020promise}
and iGN~\cite{tran2020stochastic}.
The benefits of having the exact solution are 
analyzed in~\cite{bollapragada2019exact}
with exact Newton methods enjoying faster 
convergence rates than inexact ones.
Our experiments (Section~\ref{sec:experiments}) 
also demonstrate that the exact Gauss-Newton solver (EGN) 
consistently outperforms the inexact version (SGN) across 
the majority of the problems. 
The complexity of Algorithm~\ref{alg:slm_direction_finder} 
is dominated by the matrix 
multiplication $\mathbf{J}\mathbf{J}^{\top}$
that costs $\mathcal{O}\left( b^2 c^2 d \right)$
as well as solving the linear system of size $bc \times bc$
with complexity $\mathcal{O}\left( b^3 c^3 \right)$. 
Assuming $d > bc$,  
the overall complexity of Algorithm~\ref{alg:slm_direction_finder}
is $\mathcal{O}\left(b^2 c^2 d \right)$.

A common alternative way to solve the 
Levenberg-Marquardt linear system~\eqref{eq:lm_linear_system}
exactly is to apply the SMW identity, 
as in~\cite{ren2019efficient,brust2021nonlinear,hong2020stochastic,arbel2024rethinking}.
The SMW identity states
\begin{align}\label{eq:swm_identity}
\left( \bb{A} + \bb{U} \bb{C} \bb{V} \right)^{-1}
= \bb{A}^{-1} - \bb{A}^{-1} \bb{U} \left( \bb{C}^{-1} + \bb{V} \bb{A}^{-1} \bb{U} \right)^{-1} \bb{V} \bb{A}^{-1}.
\end{align}
With
$\bb{A} = \lambda \bb{I}_d$, $\bb{U} = \bb{J}^\top$, 
$\bb{C} = \frac{1}{b}\bb{Q}$, and $\bb{V} = \bb{J}$,
we obtain the SMW-GN matrix inversion~\cite{ren2019efficient}:
\begin{align}
\left( \frac{1}{b}\bb{J}^{\top} \bb{Q} \bb{J} + \lambda \bb{I}_{d} \right)^{-1}
= \frac{1}{\lambda} \bb{I}_{d}
- \frac{1}{\lambda^{2}} \bb{J}^{\top}
\left( b\bb{Q}^{-1} + \frac{1}{\lambda} \bb{J} \bb{J}^{\top} \right)^{-1}
\bb{J}.
\end{align}
As with EGN, the dominant term
is $\mathbf{J}\mathbf{J}^{\top}$. However, even 
with efficient ordering of operations, there are at least two 
$\mathcal{O}\left(b^2 c^2 d\right)$ matrix multiplications
while EGN requires just one.
Beyond doubling the per-step cost, 
the additional multiplication can
introduce extra rounding error 
that gradually accumulates~\cite{higham2002accuracy}.

Another important consideration is the non-regularized 
case.
When Levenberg damping is disabled ($\lambda=0$), 
$\bb{A}$ in~\eqref{eq:swm_identity} 
is singular and the SMW inverse is undefined.
EGN remains well-posed provided $\bb{Q} \bb{J} \bb{J}^{\top}$ is 
invertible, making it applicable to pure Gauss-Newton steps.

We support our theoretical findings with the empirical comparison
of solving Equation~\eqref{eq:lm_linear_system} using both SMW and EGN methods (Table~\ref{tab:swm_egn_empirical}). 
EGN achieves up to $1.6\times$ speed-up compared to SMW on 
larger models ($d\!\ge\!10^{5}$), 
which translates into substantial training-time savings since 
the solver is invoked at every iteration.

\begin{table}[h]
\centering

\caption{
Wall-clock time (seconds, mean over 1000 runs) 
to solve~\eqref{eq:lm_linear_system} 
with SMW and EGN 
across different model sizes for $b=32$, $c=10$
on NVIDIA RTX A4000 GPU.
}
\label{tab:swm_egn_empirical}
\begin{tabular}{lccccc}
\hline
$d$        & \textbf{1K} & \textbf{10K} & \textbf{100K} & \textbf{1M} & \textbf{2M} \\
\hline
SMW        & 0.0010      & 0.0013      & 0.0044       & 0.0368      & 0.0747      \\
EGN        & 0.0008      & 0.0011      & 0.0028       & 0.0246      & 0.0517      \\
\hline
\end{tabular}

\end{table}

For completeness, we note here that
another option to solve~\eqref{eq:lm_linear_system}
exactly is through the QR factorization of $\mathbf{J}^{\top}_t$ 
(see pseudocode in Appendix~\ref{apx:algo}). 
The complexity of such an approach is 
also $\mathcal{O}\left(b^2 c^2 d\right)$, 
dominated by the economy size QR decomposition of $\mathbf{J}^{\top}$.
However, performing a QR decomposition is significantly more expensive than performing matrix-matrix multiplications
and, in practice, EGN is preferred.

The CG method is an essential part of, e.g., 
HFO~\cite{martens2010deep, kiros2013training},
SGN~\cite{gargiani2020promise},
Newton-CG~\cite{bollapragada2019exact},
LM~\cite{pooladzandi2022improving},
and Distributed Newton's method with optimal
shrinkage~\cite{zhang2023optimal}.
The complexity of CG approximately solving Equation~\eqref{eq:lm_linear_system}
is $\mathcal{O}(l d^2)$ where $l$ is the number of CG
iterations~\cite{nocedal2006numerical}.
A typical number of CG iterations ranges from $3$ in~\cite{kiros2013training} 
to $50$ in~\cite{chapelle2011improved}.
Note that, unless the number of classes $c$ is 
high (which is one of the limitations of EGN addressed in Section~\ref{sec:limitations}),
we have $ld^2 > b^2 c^2d$, such that EGN solves the system both exactly and faster than CG.

\subsection{Additional Improvements}\label{sec:improvements}

In addition to estimating the Hessian-adjusted direction, 
most state-of-the-art SOMs employ strategies 
to reduce variance of gradient and Hessian estimates,
safeguard against exploding gradients, and
dynamically adjust hyper-parameters to training steps. 
Next, we explore several enhancements to EGN, including 
momentum acceleration~\cite{kingma2014adam, kiros2013training, yao2021adahessian} 
to mitigate issue (d)--that is, the noise in Hessian estimates;
line search~\cite{curtis2020adaptive,vaswani2019painless,pooladzandi2022improving} to ensure steps are sufficiently short to decrease the loss;
and adaptive regularization~\cite{kiros2013training, pooladzandi2022improving, ren2019efficient} to achieve faster convergence 
and simplify hyper-parameter tuning.

\paragraph{Momentum} 
A significant challenge in second-order methods 
is the noise introduced in Hessian estimates due to stochastic sampling.
Consider the update direction 
$\mathbf{d}_t=-\mathbf{H}^{-1}_t\mathbf{g}_t$,
where both the approximate Hessian $\mathbf{H}_t$ 
and the gradient $\mathbf{g}_t$ are stochastic estimates of the true 
derivatives of the empirical risk~\eqref{eq:risk}.
Conditioning the gradient with a noisy 
Hessian inverse can lead to inaccurate descent directions, 
impeding convergence.
To mitigate this issue, temporal averaging techniques
(or \textit{momentum})
are employed to stabilize updates and accelerate convergence 
by combining information from previous iterations with current estimates. 
In ML applications, common momentum variants include
simple accumulation~\cite{duchi2011adaptive},
exponential moving average (EMA)~\cite{hinton2012neural},
bias-corrected EMA~\cite{kingma2014adam, yao2021adahessian}
and momentum with an extrapolation step~\cite{nesterov1983method}.

For diagonal scaling methods, including 
first-order accelerated algorithms~\cite{kingma2014adam, duchi2011adaptive, hinton2012neural} 
and SOMs~\cite{yao2021adahessian, liu2023sophia}, 
the accumulated estimates of both first and second moments 
are kept separately resulting in 
$\mathcal{O}\left(d \right)$
space complexity. 
For Gauss-Newton methods
we could alternatively reduce the variance of the Jacobian $\mathbf{J}$,
e.g., with SVRG~\cite{johnson2013accelerating},
SAGA~\cite{defazio2014saga} 
or SARAH~\cite{nguyen2017sarah}, 
which results in a space complexity of $\mathcal{O}\left(bcd \right)$.
Another option is to apply momentum to the 
descent direction $\mathbf{d}_t$, explored in~\cite{vaswani2019painless, kiros2013training},
which requires storing a vector of size $d$.  
Since we never explicitly materialize either the Hessian
or the gradient (Algorithm~\ref{alg:slm_direction_finder}), 
we follow the latter approach
with bias-corrected EMA by default.

\begin{algorithm}[t]
    \caption{EGN}
    \label{alg:egn_full}
    \begin{algorithmic}[1]
        \STATE {\bfseries Input:} training dataset $\mathcal{D}$,
        initial weights $\mathbf{w}_0$,
        initial regularizer $\lambda_0$,
        momentum strength $\beta$.

        \STATE Initialize momentum: $\mathbf{m}_0 = 0$
        
        \FOR{$t$ in $1..T$}
        \STATE Sample a mini-batch $\mathcal{B}_t$ from $\mathcal{D}$
        
        \STATE Estimate $\mathbf{r}_{t}$ and $\mathbf{J}_t$
        (e.g., via backpropagation)
        
        \STATE Find direction $\mathbf{d}_t$ via Algorithm~\ref{alg:slm_direction_finder}

        \STATE Calculate the momentum term: $\mathbf{m}_{t} \leftarrow \beta \mathbf{m}_{t-1} + (1-\beta)\mathbf{d}_t$

        \STATE Update direction: $\mathbf{d}_{t} \leftarrow \frac{\mathbf{m}_{t}}{1 - \beta^{t}}$

        \STATE Line search for $\alpha_t$ via Algorithm~\ref{alg:armijo_ls}
        
        \STATE Update weights: $\mathbf{w}_{t+1}\leftarrow\mathbf{w}_{t}+\alpha_{t}\mathbf{d}_{t}$
        
        \STATE Update $\lambda_{t+1}$  via Algorithm~\ref{alg:adaptive_lambda} 
        
        \ENDFOR
        
        \STATE \textbf{Return} $\mathbf{w}_t$
    \end{algorithmic}
\end{algorithm}

\paragraph{Line Search} 
Line search is a widely used technique 
in deterministic optimization~\cite{nocedal2006numerical} 
that iteratively adjusts the learning rate to satisfy 
some minimum criteria (e.g., Wolfe's conditions), 
ensuring adequate decrease in the loss function. 
Allowing $\alpha$ 
to automatically adapt to each training step
can significantly reduce the need for manual tuning, 
which is often time-consuming and computationally expensive. 
However, extending line search methods to 
the stochastic setting poses challenges due to the 
inherent noise in gradient estimates, 
making it difficult to guarantee the same theoretical properties 
as in the deterministic case~\cite{curtis2020adaptive}. 
Despite these challenges, 
line search has been successfully 
applied in practice to both stochastic FOMs~\cite{vaswani2019painless, paquette2020stochastic} 
and SOMs~\cite{pooladzandi2022improving, wills2021stochastic}.
In EGN we adopt the strategy proposed by~\cite{vaswani2019painless}, 
which incorporates a reset mechanism at the beginning of each search 
to minimize the computational overhead of evaluating the loss function
(Algorithm~\ref{alg:armijo_ls} in the Appendix).

\paragraph{Adaptive regularization}
Adaptive regularization techniques
for stochastic SOMs
are explored in~\cite{martens2010deep, kiros2013training, ren2019efficient, pooladzandi2022improving}
modifying the original Levenberg-Marquardt rule
to the stochastic setting.
The central idea is to track $\rho$, defined as the ratio
between the decrease in the actual 
loss function~\eqref{eq:batch_loss} and the decrease in 
the quadratic model
$\mathcal{M}\left(\Delta\mathbf{w}\right)=\mathcal{L}_{b}\left(\mathbf{y}_{t},\Phi(\mathbf{x}_{t};\mathbf{w}_{t})\right)+\mathbf{g}_{t}^{\top}\Delta\mathbf{w}+\frac{1}{2b}\Delta\mathbf{w}^{\top}\mathbf{J}_{t}^{\top}\mathbf{Q}_{t}\mathbf{J}_{t}\Delta\mathbf{w}$,
where $\Delta\mathbf{w}=\mathbf{w}-\mathbf{w}_{t}$.
This yields
\begin{align}\label{eq:rho_lm}\rho:=\frac{\mathcal{L}_{b}\left(\mathbf{y}_{t},\Phi(\mathbf{x}_{t};\mathbf{w}_{t+1})\right)-\mathcal{L}_{b}\left(\mathbf{y}_{t},\Phi(\mathbf{x}_{t};\mathbf{w}_{t})\right)}{\mathbf{g}_{t}^{\top}\Delta\mathbf{w}+\frac{1}{2b}\Delta\mathbf{w}^{\top}\mathbf{J}_{t}^{\top}\mathbf{Q}_{t}\mathbf{J}_{t}\Delta\mathbf{w}},
\end{align}
which measures the accuracy of the quadratic model. In case 
$\rho$ is small or negative, $\mathcal{M}\left(\Delta\mathbf{w}\right)$
provides inaccurate approximation and the value $\lambda$ is increased.
Conversely, if $\rho$ is large, $\lambda$ is decreased to give more weight to 
$\mathcal{M}\left(\Delta\mathbf{w}\right)$ 
(see Algorithm~\ref{alg:adaptive_lambda} in the Appendix).
As empirically found by~\cite{kiros2013training},
compared to the deterministic LM
the increase/decrease coefficients need 
to be less aggressive to reduce the oscillations of $\lambda_t$.

Algorithm~\ref{alg:egn_full} 
incorporates all the improvements discussed above,  
effectively addressing issues (a)-(d) associated with SOMs 
and eliminating the need for manual hyper-parameter tuning.

\section{Convergence Analysis}\label{sec:theory}
In this section, we analyze the convergence of EGN in the general non-convex setting.

In EGN, we consider the sequence of iterates $\{\bbw{t}\}_{t\ge1}$ where each $\bbw{t}$ is computed via \eqref{eq:gd_update} with $\bb{d}_t \equiv \bb{d}_t^{\mathrm{LM}}$. We aim to minimize the function $\calL_N(\bb{w})$ using, for each realization $\xi\sim P_\xi$, the Hessian estimator $\bbH(\bb{w},\xi)$ and the gradient estimator $\bbg(\bb{w},\xi)$. Then, at each iteration $t$, the mini-batch estimates $\bbg_t$ and $\bbH_t$ of the gradient and Hessian are
\begin{align}\label{eq:d-estimates}
\bbg_t = \frac{1}{b}\sum_{\xi_i\in\calB_t}\bbg(\bb{w},\xi_i) \coloneqq \frac{1}{b}\bbJ_t^\top\bbr_t, \\
\bbH_t = \frac{1}{b}\sum_{\xi_i\in\calB_t}\bbH(\bb{w},\xi_i) \coloneqq \frac{1}{b}\bbJ_t^\top\bbQ_t\bbJ_t.
\end{align}
In our analysis, we make use of the following assumptions.
\begin{assumption}\label{ass:regularity}
    The function $\calL_N$ is lower-bounded on its domain, \ie, $-\infty < \calL_N^* \coloneqq \inf\limits_{\bb{w}\in\rr^d} \calL_N(\bb{w})$. In addition, $\calL_N$ is twice differentiable with Lipschitz continuous first-order derivatives, \ie, $\exists L_1\in\mathbb{R}$ such that
	\begin{align}
		&\norm{\nabla \calL_N(\bar{\bb{w}}) - \nabla \calL_N(\tilde{\bb{w}})} \le L_1\norm{\bar{\bb{w}} - \tilde{\bb{w}}}, \quad \forall \bar{\bb{w}}, \tilde{\bb{w}} \in \rr^d.
	\end{align}
\end{assumption}
\begin{assumption}\label{ass:variance}
	At any iteration $t$, $\bbg(\bbw{t},\xi)$ is an unbiased estimator of $\nabla \calL_N(\bbw{t})$, \ie,
	\begin{align}
		\EE_\xi[\bbg(\bbw{t},\xi)] = \nabla \calL_N(\bbw{t}).\label{eq:unb-g}
	\end{align}
	Moreover, we have
	\begin{align}
		\EE_\xi\sqbracket{\norm{\bbg(\bbw{t},\xi) - \nabla \calL_N(\bbw{t})}^2} \le \sigmag^2,
	\end{align}
	where $\sigmag>0$ is a variance parameter. 
\end{assumption}
\begin{assumption}\label{ass:ggn}
	At any iteration $t$, $\ubar{\kappa}\bb{I} \le \bbQ_t \le \bar{\kappa}\bb{I}$ with $\bar{\kappa} \ge \ubar{\kappa} \ge 0$. 
    Additionally, $\exists \bar{\sigma}, \ubar{\sigma}$ satisfying $0<\ubar{\sigma}\le \bar{\sigma}$ such that $\ubar{\sigma} \le \norm{\bbJ_t} \le \bar{\sigma}$ for all $t\ge 0$.
\end{assumption}
\begin{assumption}\label{ass:ggn-extra}
	At any iteration $t$, and for any random matrix $\bb{B}_t$ satisfying $\bb{B}_t \succeq \mu \bb{I}$ with $\mu>0$, $\EE\left[\langle \nabla \calL_N(\bbw{t}), \bb{B}_t\bbg_t \rangle|\bbw{t}\right] \ge \mu K \|\nabla \calL_N(\bbw{t})\|\|\EE[ \bbg_t|\bbw{t}]\|$, where $K=(3(b\lambda_t + \bar{\kappa}\bar{\sigma}^2))/(5b\lambda_t)$.
\end{assumption}

These assumptions are standard in stochastic optimization literature \cite{bottou2018optimization, ghadimi2013stochastic, gower2020variance},
and they hold naturally or can be easily enforced for typical neural network architectures and loss functions used in machine learning practice. 
Notably, the variance of the stochastic gradient estimator (Assumption~\ref{ass:variance}) is more directly controllable from a practitioner's perspective by selecting a sufficiently large batch size or adding a momentum term that uses past gradients 
to inform the direction of update 
(see, \eg, \cite{bottou2018optimization, gower2020variance, defazio2019ineffectiveness}). 
We also remark that the matrix $\bb{B}_t$ in Assumption~\ref{ass:ggn-extra} may be interpreted as a preconditioner. 
Depending on the nature of its conditional correlation 
with $\bbg_t$ given $\bbw{t}$, a large batch size can enhance the practicality of the assumption. 
In practice, either this conditional correlation is nonexistent 
(see the comment on \cite[Assumption 4.3(b)]{bottou2018optimization}), 
or the contribution of $\bbg_t$ is inversely scaled by the (large) batch size as in the proof of Lemma~\ref{thm:desc} below.

In Lemma \ref{thm:desc}, we prove a descent lemma for EGN under the given conditions.
\begin{lemma}\label{thm:desc}
Let $\{\bbw{t}\}$ be the sequence of iterates generated by \eqref{eq:gd_update} with $\bb{d}_t \equiv \bb{d}_t^{\mathrm{LM}}$ and let Assumptions \ref{ass:regularity}--\ref{ass:ggn} hold. Suppose there exists $\alpha^{\text{max}} > 0$ such that $\alpha_t \le \alpha^{\text{max}}$ for all $t$ in Algorithm~\ref{alg:armijo_ls}. Then,
\begin{align}
    \EE[\calL_N(\bbw{t+1})|\bbw{t}] &\le \calL_N(\bbw{t}) 
    - \frac{\alpha_t}{2}\|\nabla \calL_N(\bbw{t})\|^2 + \frac{3\sigmag^2 \alpha_t^2}{10 b \lambda_t\alpha^{\text{max}}}.
    \label{eq:desc}
\end{align}
\end{lemma}
\begin{proof}
From Assumption \ref{ass:regularity}, we have
\begin{align}
    &\calL_N(\bbw{t+1}) \le \calL_N(\bbw{t}) + \langle\nabla \calL_N(\bbw{t}), \bbw{t+1} - \bbw{t}\rangle + \frac{L_1}{2}\norm{\bbw{t+1} - \bbw{t}}^2 \nonumber \\
    &= \calL_N(\bbw{t}) - \alpha_t \left\langle \nabla \calL_N(\bbw{t}), 
    \left(\frac{1}{b} \bbJ_t^\top \bbQ_t \bbJ_t + \lambda_t \bb{I}\right)^{-1} \bbg_t \right\rangle + \frac{\alpha_t^2 L_1}{2} 
    \norm{\left(\frac{1}{b} \bbJ_t^\top \bbQ_t \bbJ_t + \lambda_t \bb{I}\right)^{-1} \bbg_t}^2.
    \label{eq:desc-1}
\end{align}
Next, we set $\bb{B}_t=\left(\frac{1}{b}\bbJ_t^\top\bbQ_t\bbJ_t+\lambda_t\bb{I}\right)^{-1}$ 
in Assumption~\ref{ass:ggn-extra}. 
From Assumption~\ref{ass:ggn}, we can show that 
$\bb{B}_t \succeq \mu\bb{I}$ with $\mu \coloneqq b/(b\lambda_t + \bar{\kappa}\bar{\sigma}^2)$. Using this result and Assumption~\ref{ass:ggn} in \eqref{eq:desc-1}, we obtain
\begin{align}
    \calL_N(\bbw{t+1}) &\le \calL_N(\bbw{t}) 
    - \alpha_t \left\langle \nabla \calL_N(\bbw{t}), \bb{B}_t \bbg_t \right\rangle + \frac{b^2 \alpha_t^2 L_1}{2(b\lambda_t + \ubar{\kappa} \ubar{\sigma}^2)^2} \|\bbg_t\|^2.
    \label{eq:desc-2}
\end{align}
Notice that by Assumption \ref{ass:variance}, we have $\EE[ \bbg_t|\bbw{t}] = \nabla \calL_N(\bbw{t})$. 
Hence, the inequality in Assumption~\ref{ass:ggn-extra} can be written as $\EE\left[\langle \nabla \calL_N(\bbw{t}), \bb{B}_t\bbg_t \rangle|\bbw{t}\right] \ge \mu K \|\nabla \calL_N(\bbw{t})\|^2$. Taking conditional expectation on both sides of \eqref{eq:desc-2} with respect to $\xi$ and using Assumption \ref{ass:ggn-extra}, we get
\begin{align}
    \EE[\calL_N(\bbw{t+1})|\bbw{t}] &\le \calL_N(\bbw{t}) 
    - \frac{3\alpha_t}{5\lambda_t} \|\nabla \calL_N(\bbw{t})\|^2 + \frac{b^2 \alpha_t^2 L_1}{2(b\lambda_t + \ubar{\kappa} \ubar{\sigma}^2)^2} \EE[\|\bbg_t\|^2|\bbw{t}].
    \label{eq:desc-3}
\end{align}
Using Assumption~\ref{ass:variance}, we have
\begin{align}
    \EE[\|\bbg_t\|^2|\bbw{t}] &= \EE[\|\bbg_t - \nabla \calL_N(\bbw{t}) + \nabla \calL_N(\bbw{t})\|^2|\bbw{t}] \nonumber \\
    &= \EE[\|\nabla \calL_N(\bbw{t})\|^2|\bbw{t}] \nonumber \\ 
    &\qquad + 2\EE[\langle \bbg_t - \nabla \calL_N(\bbw{t}), \nabla \calL_N(\bbw{t})\rangle | \bbw{t}] + \EE[\|\bbg_t - \nabla \calL_N(\bbw{t})\|^2|\bbw{t}] \nonumber \\
    &= \|\nabla \calL_N(\bbw{t})\|^2 + \EE[\|\bbg_t - \nabla \calL_N(\bbw{t})\|^2|\bbw{t}] \nonumber \\
    &\le \|\nabla \calL_N(\bbw{t})\|^2 + \frac{\sigma_g^2}{b}.
    \label{eq:desc-4}
\end{align}
Now, using \eqref{eq:desc-4} in \eqref{eq:desc-3}, we obtain
\begin{align}
    \EE[\calL_N(\bbw{t+1})|\bbw{t}] &\le \calL_N(\bbw{t}) - \left(\frac{3\alpha_t}{5\lambda_t} 
    - \frac{b^2 \alpha_t^2 L_1}{2(b\lambda_t + \ubar{\kappa} \ubar{\sigma}^2)^2}\right) 
    \|\nabla \calL_N(\bbw{t})\|^2 + \frac{b \sigmag^2 \alpha_t^2 L_1}{2(b\lambda_t + \ubar{\kappa} \ubar{\sigma}^2)^2}.
    \label{eq:desc-5}
\end{align}
By choosing $\alpha_t$ to be sufficiently small, we can set 
$\alpha^{\text{max}} = 1 / c_1$ in Algorithm~\ref{alg:armijo_ls}, 
where
$c_1 = (5b^2\lambda_tL_1)/(3(b\lambda_t + \ubar{\kappa}\ubar{\sigma}^2)^2$.
With this selection, inequality~\eqref{eq:desc-5} holds, 
which completes the proof.

\end{proof}
Assumption~\ref{ass:ggn-extra} is necessary to prove the result in Lemma \ref{thm:desc} due to the correlation between $\bb{B}_t$ and $\bbg_t$. If these quantities were not correlated, then one could exploit Assumptions~\ref{ass:variance}--\ref{ass:ggn}: by taking the expected value and using~\eqref{eq:unb-g}, one could then bound $\bb{B}_t\succeq \mu \bb{I}$ to obtain the desired bound with $K=1$. Clearly, that would yield different constants in our next results, but their nature would remain unaltered. Note that this suggests alternative versions of EGN, in which the correlation is removed by construction at the expense of an increased computational burden. A thorough investigation of such schemes is beyond the scope of this paper.

Unlike classical stochastic quasi-Newton and SGD ($\bbH_t = \bb{I}_d$) methods, the loss decrease condition of EGN has an explicit nonlinear dependence on the batch size $b$. In the regime $b \gg 1$, the influence of the variance parameter $\sigma_g$ diminishes.

We prove next the following result about the convergence of EGN.
\begin{theorem}\label{thm:iter-complexity-1}
    Let the assumptions in Lemma \ref{thm:desc} hold, and assume $\lambda_t\equiv \lambda$ is fixed for all $t$. Let  $\alpha_t=\frac{\alpha_0}{(t+1)^a}$ for some $0<\alpha_0<1$, $\frac{1}{2}<a<1$.
    Then, the loss gradient approaches $0$ in expectation as the iteration count $T\to\infty$.
\end{theorem}
\begin{proof}
Following the proof of Lemma~\ref{thm:desc}, 
we fix $\alpha^{\text{max}} = 1/c_1$ in Algorithm~\ref{alg:armijo_ls}, 
with $c_1 = (5b^2\lambda_tL_1)/(3(b\lambda_t + \ubar{\kappa}\ubar{\sigma}^2)^2$, and take
the conditional expectation on both sides of \eqref{eq:desc}. 
This yields
\begin{align}
    \EE[\calL_N(\bbw{t+1})] &\le \EE[\calL_N(\bbw{t})] 
    - \frac{1}{2}\alpha_t\EE[\|\nabla \calL_N(\bbw{t})\|^2] + C\alpha_t^2,
    \label{eq:alt-2}
\end{align}
where $C:=\frac{3\sigmag^2 c_1}{10 b \lambda_t}$. 
Summing \eqref{eq:alt-2} over $t=0,1,\ldots,T-1$, we obtain
\begin{align}
    \EE[\calL_N(\bbw{T})] &\le \EE[\calL_N(\bbw{0})] 
    - \frac{1}{2}\sum_{t=0}^{T-1}\alpha_t\EE[\|\nabla \calL_N(\bbw{t})\|^2] + C\sum_{t=0}^{T-1}\alpha_t^2.
    \label{eq:alt-3}
\end{align}
Assuming that the learning rates are given by
$\alpha_t=\frac{\alpha_0}{(t+1)^a}$ for some $0<\alpha_0<1$, $\frac{1}{2}<a<1$, we have
\begin{align*}
    \sum_{t=0}^{T-1}\alpha_t^2 = D < \infty.
\end{align*}
Next, consider a random variable $\bb{z}_T$ satisfying $P[\bb{z}_T=\bbw{t}]=\frac{1}{T},$ for all $t$. Consequently,
\begin{align*}
    \EE[\|\nabla \calL_N(\bb{z}_T)\|^2] = \frac{1}{T}\sum_{t=0}^{T-1} \EE[\|\nabla \calL_N(\bbw{t})\|^2],
\end{align*}
such that, using also $\alpha_{t_2}\leq \alpha_{t_1}$ for all $t_2\geq t_1$,~\eqref{eq:alt-3} becomes
\begin{align*}
    \EE[\calL_N(\bbw{T})] 
     &\le \EE[\calL_N(\bbw{0})] 
    - \alpha_{T-1}\frac{1}{2}T\EE[\|\nabla \calL_N(\bb{z}_T)\|^2]+ CD.
\end{align*}
Then, using $\EE[\calL_N(\bbw{T})]\geq \calL_N^\star$ we have
\begin{align*}
    \EE[\|\nabla \calL_N(\bb{z}_T)\|^2] &\leq 2\frac{\EE[\calL_N(\bbw{0})] - \calL_N^*}{\alpha_{T-1}T} + \frac{2CD}{\alpha_{T-1}T} = 2\frac{\EE[\calL_N(\bbw{0})] - \calL_N^*}{\alpha_0T^{1-a}} + \frac{2CD}{\alpha_0T^{1-a}}.
\end{align*}
Taking the limit for $T\to\infty$ we finally have
\begin{align*}
    \lim_{T\to\infty}\EE[\|\nabla \calL_N(\bb{z}_T)\|^2] &\leq 0.
\end{align*}
\end{proof}

\section{Experiments}\label{sec:experiments}

We conduct a series of experiments to measure the 
performance of EGN on several supervised learning
and reinforcement learning tasks.

We select five baseline solvers:
SGD~\cite{robbins1951stochastic},
Adam~\cite{kingma2014adam},
SGD with momentum and Gradient Activation Function (GAF)~\cite{liu2021activated},
a Quasi-Newton solver (SQN)~\cite{bottou2018optimization, byrd2016stochastic, schraudolph2007stochastic},
and SGN~\cite{gargiani2020promise}.  
SGD acts as a most basic baseline with the computationally cheapest update; 
Adam is a widely used accelerated FOM for training DNNs;
GAF is a recent first-order variant 
that reports faster convergence on deep networks;
SGN is an \textit{inexact} Gauss-Newton solver against which we  
evaluate the practical advantages of solving the system~\eqref{eq:linear_system} \textit{exactly};
and SQN acts as an alternative to Gauss-Newton that approximates  
the Hessian via low-rank updates.
Since first-order methods typically require fewer computations 
per iterate, 
in order to obtain a fair comparison we monitor the wall time 
instead of the number of iterations.  
The learning rates are selected as the best performing 
$\alpha$ after a grid search in the logspace 
$\alpha \in [10^{-9}, 1]$. Additionally, for SGN
we search for the optimal ``number of CG iterations'' 
within the set $\{3, 5, 10, 20, 50\}$. For EGN 
we introduce two extra hyper-parameters: ``line search''
$\{ \text{True}, \text{False} \}$
and ``momentum'' $\{ 0.0, 0.9 \}$. The best performing 
sets of hyper-parameters 
as well as detailed description of the datasets
are available in Appendix~\ref{apx:experiments}.
The size of the mini-batch for all problems is $128$. 
All the experiments are conducted 
on the Tesla T4 GPU  
in the Google Colab environment 
with float32 precision.

\subsection{Supervised Learning}

For the regression task, we select three datasets: California Housing~\cite{ds_california_housing}, Superconductivity~\cite{ds_superconductivty}, and Diamonds~\cite{ds_diamonds}
with $20640$, $21263$, and $53940$ training samples, respectively. 
For classification, we use the IMDB Reviews dataset~\cite{ds_imdb}
containing $25000$ instances of movie reviews. 
Across all problems, 
the model $\Phi$
is a Feedforward Neural Network (FFNN) with three dense 
layers of $32$, $64$ and $32$ units followed by 
the ReLU activation function with a total of $4449$ parameters ($d=4449$).  
The loss function during the training is a least-squares 
loss~\eqref{eq:loss_mse} for regression 
and softmax cross-entropy loss~\eqref{eq:loss_ce} for classification. 
The datasets are split into training and test sets
in the proportion of $90/10$ percent. Numerical features 
are scaled and categorical features are 
one-hot encoded. 
To measure the performance, we plot
the evolution of the evaluation metric 
on unseen data (test set) with respect to 
wall time (in seconds). The evaluation metric 
is Root Mean Squared Error (RMSE) for regression
and accuracy for classification.

\begin{figure}[t] 
    \includegraphics[width=0.7\textwidth]{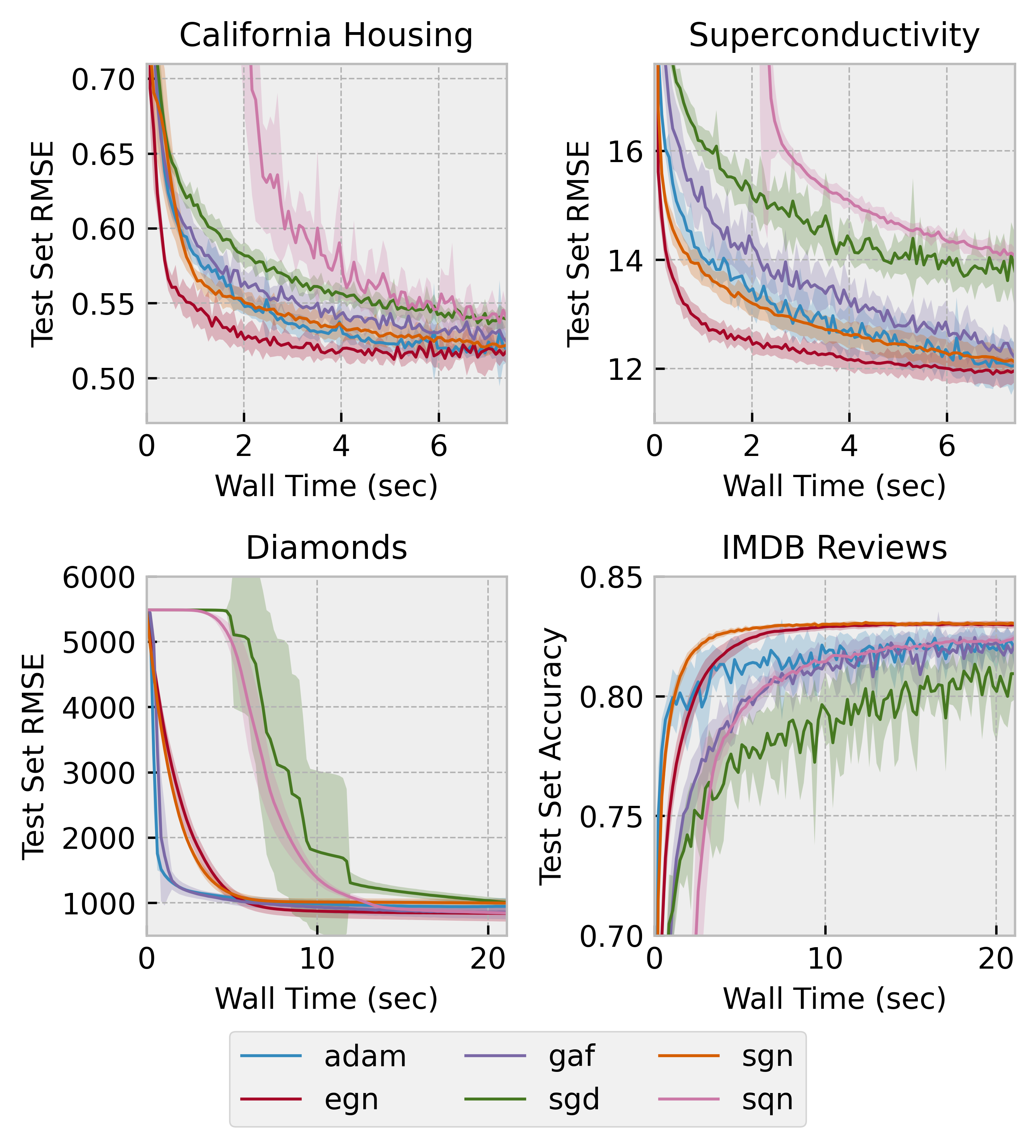}
    \centering
    \caption{Learning curves on the test set for
        SGD, Adam, GAF, SQN, SGN and EGN. 
        The shaded
        area represents $\pm1$ standard deviation 
        around the mean (thick line) for $10$ seeds.}
    \label{fig:sl}
\end{figure}

The results are presented in Figure~\ref{fig:sl} and Table~\ref{tab:perf_sl}.
On all but the IMDB Reviews dataset EGN has achieved 
both faster convergence and lower test set error than any other
optimization algorithm. On IMDB Reviews SGN is faster than EGN,
however, the two solvers achieve the same accuracy after 
reaching convergence.

\begin{table*}[t]
\centering
\caption{Performance after Training Completion (Supervised Learning)}
\label{tab:perf_sl}
\resizebox{\textwidth}{!}{
\begin{tabular}{@{}lcccc@{}}
\toprule
\textbf{Optimizer} & \textbf{California Housing} & \textbf{Superconduct} & \textbf{Diamonds} & \textbf{IMDB Reviews} \\ \midrule
SGD & 0.539 $\pm$ 0.006 & 13.788 $\pm$ 0.698 & 1008.936 $\pm$ 54.940 & 0.809 $\pm$ 0.011 \\
Adam & 0.519 $\pm$ 0.009 & 12.052 $\pm$ 0.381 & 947.258 $\pm$ 130.079 & 0.820 $\pm$ 0.008 \\
GAF & 0.524 $\pm$ 0.006 & 12.193 $\pm$ 0.293 & 857.817 $\pm$ 109.033 & 0.822 $\pm$ 0.006 \\
EGN & \textbf{0.518 $\pm$ 0.009} & \textbf{11.961 $\pm$ 0.207} & \textbf{840.500 $\pm$ 126.444} & \textbf{0.830 $\pm$ 0.001} \\
SGN & 0.522 $\pm$ 0.007 & 12.121 $\pm$ 0.196 & 998.688 $\pm$ 61.489 & \textbf{0.830 $\pm$ 0.001} \\
SQN & 0.539 $\pm$ 0.008 & 14.070 $\pm$ 0.141 & 844.951 $\pm$ 48.109 & 0.824 $\pm$ 0.001 \\
\bottomrule

\end{tabular}
}

\end{table*}

\subsection{Reinforcement Learning}

We demonstrate the application of EGN to reinforcement learning
in two scenarios: 
continuous action spaces with Linear-Quadratic Regulator (LQR) 
and discrete action spaces using Deep Q-Network (DQN)~\cite{mnih2013playing}. 

\paragraph{Learning LQR Controllers}

Given a discrete time-invariant linear 
system with continuous states and actions, 
and a quadratic reward function 
our task is to learn the optimal value function $v^{*}(s)$
and the optimal policy $\pi^{*}(s)$
such that we maximize the cumulative return.
Such problems can be solved in a data-driven fashion with
the policy iteration procedure~\cite{bradtke1994adaptive} 
(outlined in Appendix~\ref{apx:lqr}). 
It is well known that the optimal value function is quadratic 
and the optimal policy function is linear~\cite{hazan2022introduction}.
Consequently, we define $\Phi$ as a quadratic function of states and actions.
We track the norm of the difference between the 
optimal LQR controller calculated analytically knowing 
the system matrices and the learned weights of the model.

We select two linear systems from the 
Compleib set of benchmarks \cite{leibfritz2006compleib}. 
The first system is a deterministic 
model of a binary 
distillation tower (BDT) \cite{davison1990benchmark},
and the second one represents the 
linearized vertical
plane dynamics of an aircraft (UAV) with noise \cite{hung1982multivariable}.
The results are displayed in the top two charts 
of Figure~\ref{fig:rl}
and Table~\ref{tab:perf_rl}.
Both EGN and SGN outperform first-order methods
by a considerable margin,
with EGN enjoying slightly faster convergence in both cases,
while SGN achieves a marginally lower error on the
stochastic LQR upon reaching convergence.

\paragraph{Reinforcement Learning with DQN} 
Adopting the problem formulation of~\cite{mnih2013playing},
we aim to learn the weights of a neural network that represent
a $Q$-value function $q(s,a)$ that maps states $s$ 
and actions $a$ into scores ($Q$-values) for
a discrete set of actions.
Once training is complete, the optimal policy 
is formed by calculating the $Q$-value of each action and 
choosing the highest-scoring action.

\begin{figure}[t] 
    \includegraphics[width=0.7\textwidth]{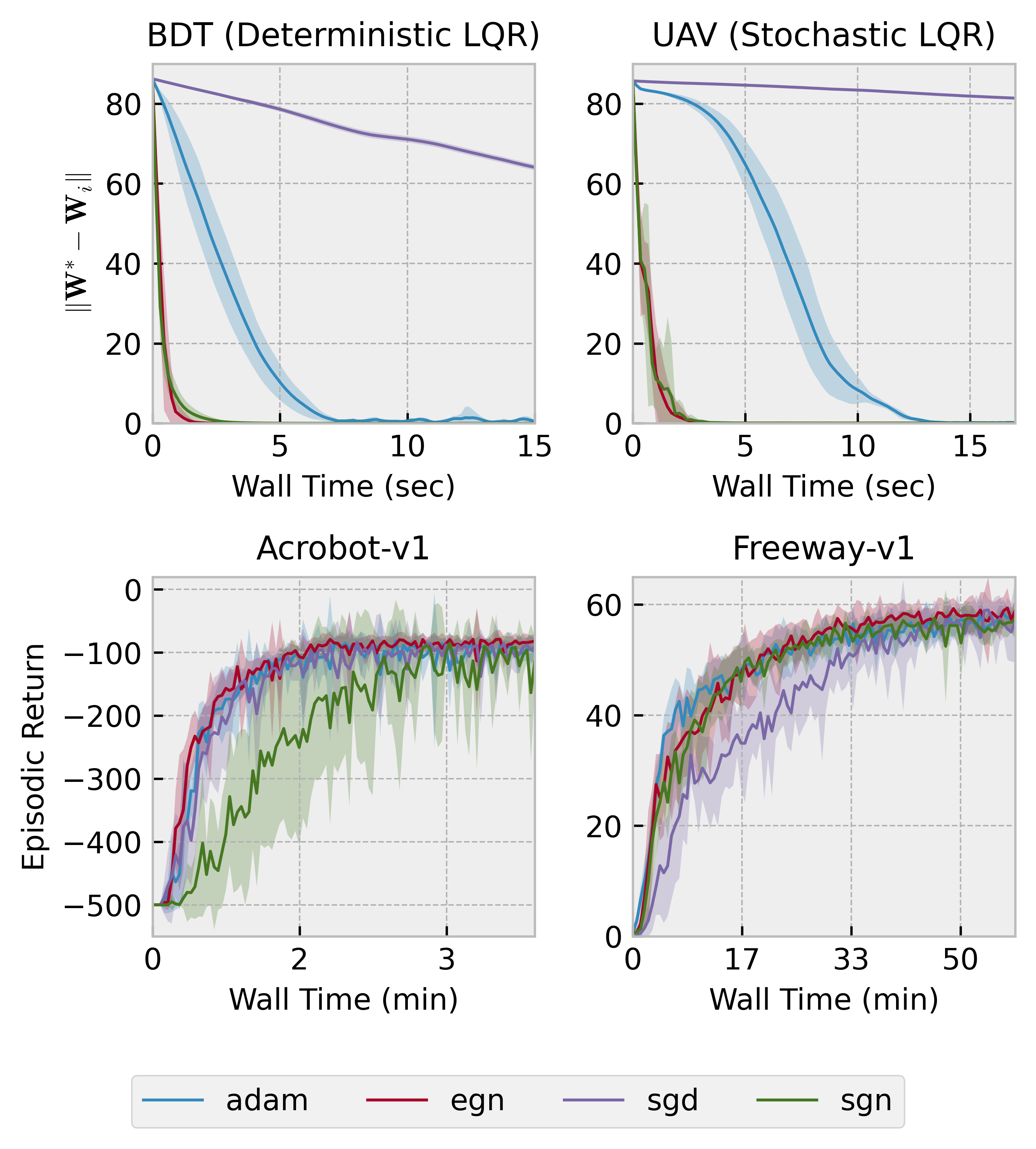}
    \centering
    \caption{Learning curves for SGD, Adam, SGN and EGN. The shaded area represents $\pm 1$ standard deviation around the mean return (thick line) for $10$ seeds.}
    \label{fig:rl}
\end{figure}

\begin{table*}[h]
\centering
\caption{Performance after Training Completion (Reinforcement Learning)}
\label{tab:perf_rl}
\resizebox{\textwidth}{!}{
\begin{tabular}{@{}lcccc@{}}
\toprule
\textbf{Optimizer} & \textbf{BDT} & \textbf{UAV} & \textbf{Acrobot-v1} & \textbf{Freeway-v1} \\ \midrule
SGD    & 64.058 $\pm$ 0.768 & 81.378 $\pm$ 0.306 & -90.962 $\pm$ 13.484 & 56.996 $\pm$ 7.359 \\
Adam   & 0.443 $\pm$ 0.297 & 0.157 $\pm$ 0.100 & -98.489 $\pm$ 20.827 & 55.967 $\pm$ 2.251 \\
EGN    & \textbf{0.000 $\pm$ 0.000} & 0.043 $\pm$ 0.031 & \textbf{-81.796 $\pm$ 11.778} & \textbf{58.916 $\pm$ 1.946} \\
SGN    & \textbf{0.000 $\pm$ 0.000} & \textbf{0.033 $\pm$ 0.019} & -114.052 $\pm$ 39.914 & 57.424 $\pm$ 3.611 \\ \bottomrule
\end{tabular}
}
\end{table*}

We build upon CleanRL~\cite{huang2021cleanrl} framework for running 
RL experiments, selecting two environments: Acrobot-v1 and Freeway-v1.
Acrobot-v1 is an OpenAI gym~\cite{openai_gym} environment 
with a $6$-dimensional state vector and a set of $3$ discrete actions
where the goal is to swing the free end of the connected joints 
above a given height in as few steps as possible.
Freeway-v1 is a MinAtar~\cite{young19minatar} environment that 
emulates the original Freeway Atari game. The state is 
represented by a $10 \times 10$ image
and there are $3$ discrete actions available. 
For Acrobot-v1 the network $\Phi$ is a FFNN 
with three dense layers of $32$, $64$ and $32$ units followed by 
the ReLU activation function with a total of $4515$ parameters.
For Freeway-v1 we design a compact CNN, 
comprising of a convolutional layer with 
sixteen $3 \times 3$ filters and ReLU activation, 
followed by flattening, a $64$-unit dense layer with ReLU, 
and a final dense layer outputting $Q$-values for 
all actions ($d=103683$).

The cumulative returns from each completed episode are 
recorded and 
displayed in the bottom two charts of Figure~\ref{fig:rl}.
The results for both Acrobot-v1 and Freeway-v1 show no distinct 
advantage among the solvers, 
as they all reach similar episodic returns. 
We notice, however, that EGN slightly outperforms 
other optimizers by achieving a higher return level 
at convergence (see Table~\ref{tab:perf_rl}).

\subsection{Limitations}\label{sec:limitations}

\paragraph{Explicit gradients} 
Unlike first-order methods that rely on the average 
gradient of the batch loss, 
Gauss-Newton methods require the full Jacobian matrix, which contains the gradients of each sample. 
As a result, backpropagation for EGN is more 
time-consuming than for FOMs.
Moreover, this can lead to increased GPU memory usage, 
especially with high-dimensional parameter vectors.

\paragraph{Large batch sizes} 

In our experiments we observed that 
the cost of computing the derivatives and the subsequent 
cost of computing the step 
(Algorithm~\ref{alg:slm_direction_finder}) 
are comparable.
However, for large batch sizes ($b>128$) we observed that the 
computational times increased 
significantly. 
Figure~\ref{fig:batch} quantifies how the 
direction calculation stage
begins to dominate the update time as the batch size grows.
While EGN remains an efficient drop-in replacement for
first-order methods within the commonly-used range
\(32\le b\le128\),
scaling to very large batches will require additional techniques. 
We suspect this to be related to hardware limitations
and leave a systematic study of such
mitigations to future work.

\begin{figure}[t] 
    \centering
    \includegraphics[width=0.5\textwidth]{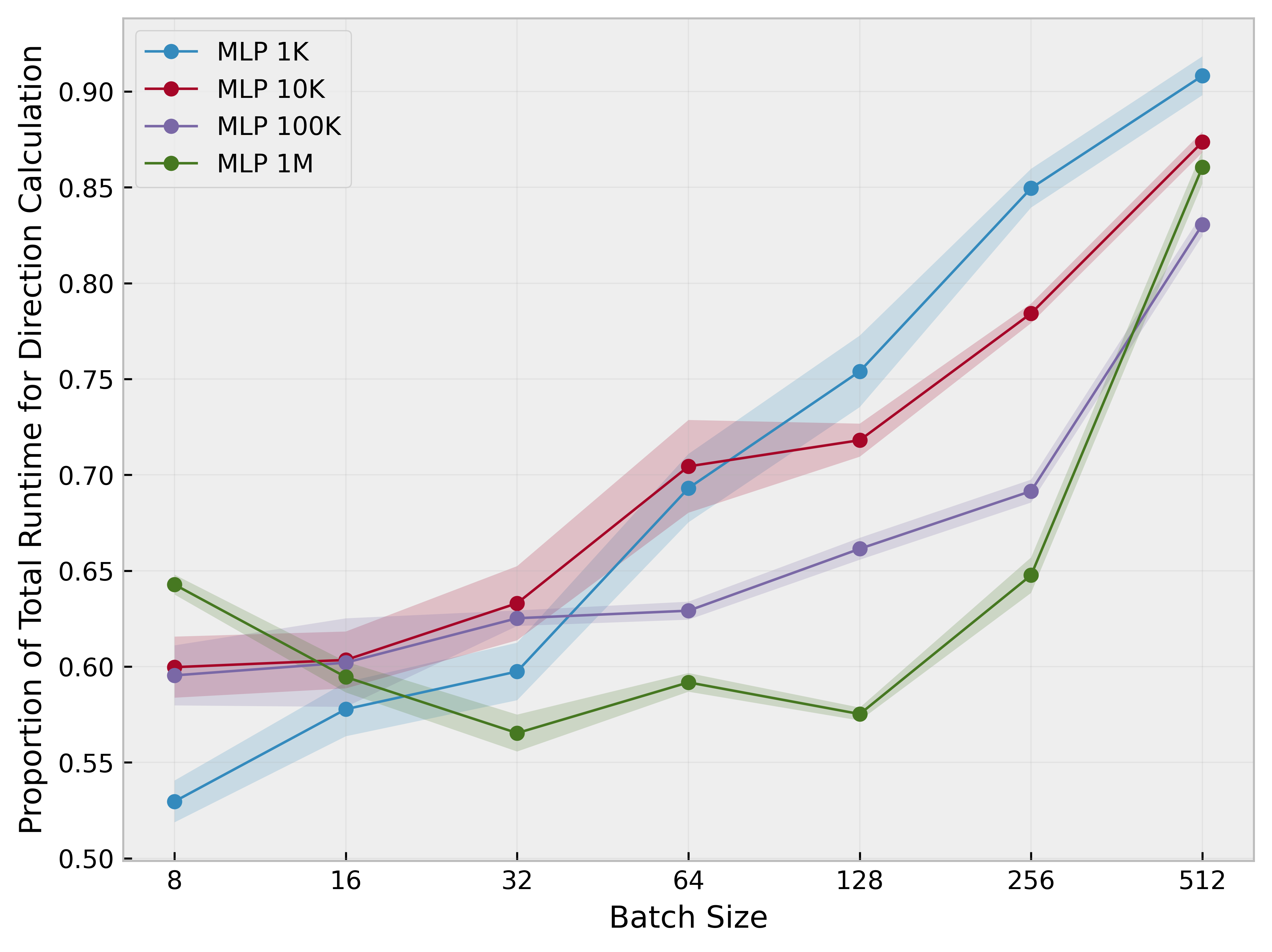}
    \caption{
Proportion of total update time spent in Algorithm~\ref{alg:slm_direction_finder} 
as a function of batch size for fully-connected neural networks (MLPs) of various sizes. 
Curves show the mean over 1000 runs; 
shaded regions denote $\pm 1$ standard deviation. 
Absolute wall-clock times are reported in Appendix~\ref{apx:limitations}.
}
    \label{fig:batch}
\end{figure}

\paragraph{Large number of classes for multi-class classification} By incorporating the softmax function
into the loss equation~\eqref{eq:loss_ce}
we introduce coupling between the 
individual outputs of $\Phi$ in the denominator 
$\sum_{j=1}^{c} e^{\mathbf{z}_{i,j}}$,
which makes
the Jacobian $\mathbf{J}$ 
a dense $bc \times d$ matrix.
Speeding up the calculation 
of $\mathbf{J}$ remains an open research question
and constitutes a major obstacle in using Gauss-Newton methods 
for tasks involving a large number of classes,
e.g., LLM pre-training.
One possible solution could consist in  moving the softmax function 
directly into the model $\Phi$ as the last layer of the network 
and then only computing the Jacobian with respect to the correct class, 
resulting in $\mathbf{J} \in \mathbb{R}^{b\times d}$. This approach, however, yields a Hessian approximation that becomes singular close to the solution, which results in numerical instabilities 
that hinder convergence~\cite{grosse2021taylor}.
Another possibility suggested in~\cite{ozyildirim2021levenberg} consists of
replacing the softmax cross-entropy loss with 
the multi-class hinge loss. 
Although the computation becomes faster for the hinge loss, 
empirical evidence shows that 
the test set accuracy upon training completion
is higher for the CE loss.
Finally, a promising approach 
is the Gauss-Newton-Bartlett (GNB) estimator proposed 
by~\cite{liu2023sophia}, which replaces the exact LM Hessian with an approximation obtained by sampling the subset of predicted labels. 

\section{Conclusion}

We presented the EGN algorithm, a stochastic second-order method
that efficiently estimates the descent direction by using a low-rank 
Gauss-Newton Hessian approximation and 
leveraging the Duncan-Guttman matrix identity.
We demonstrated that EGN solves 
the system $\mathbf{H}^{\text{LM}}\mathbf{d}=-\mathbf{g}$
exactly with less computational burden than 
other exact Gauss-Newton methods, as well as inexact   
methods that rely on conjugate gradient iterates. 
We also proved that under mild assumptions our algorithm converges in expectation.
Our empirical results show that EGN consistently matches or exceeds 
the generalization performance 
of well-tuned SGD, Adam, GAF, SQN, and SGN optimizers across various 
supervised and reinforcement learning tasks.

Future work will focus 
on addressing the shortcomings of EGN in classification problems
with a large number of classes. A promising direction is to approximate
the Gauss-Newton Hessian matrix to avoid computing the full 
Jacobian of the network, e.g., using techniques such as the Gauss-Newton-Bartlett estimator~\cite{liu2023sophia}.
Another direction is to study the performance of EGN on  
larger datasets and more complex models.

\section*{Acknowledgments}

This work was partially funded by the European Union 
(ERC Advanced Research Grant COMPACT, No. $101141351$). 
Views and opinions expressed are however those of the authors only and do not necessarily reflect those of the European Union or the European Research Council. 
Neither the European Union nor the granting authority can be held responsible for them.


\bibliographystyle{plain}

\bibliography{references}

\begin{thebibliography}{10}

\bibitem{adeoye2021sc}
Adeyemi~D Adeoye and Alberto Bemporad.
\newblock Sc-reg: Training overparameterized neural networks under
  self-concordant regularization.
\newblock Technical report, IMT School for Advanced Studies Lucca, 2021.

\bibitem{adeoye2023score}
Adeyemi~D Adeoye and Alberto Bemporad.
\newblock {SCORE}: approximating curvature information under self-concordant
  regularization.
\newblock {\em Computational Optimization and Applications}, 86(2):599--626,
  2023.

\bibitem{agarwal2017second}
Naman Agarwal, Brian Bullins, and Elad Hazan.
\newblock Second-order stochastic optimization for machine learning in linear
  time.
\newblock {\em The Journal of Machine Learning Research}, 18(1):4148--4187,
  2017.

\bibitem{amari1998natural}
Shun-Ichi Amari.
\newblock Natural gradient works efficiently in learning.
\newblock {\em Neural computation}, 10(2):251--276, 1998.

\bibitem{arbel2024rethinking}
Michael Arbel, Romain Menegaux, and Pierre Wolinski.
\newblock Rethinking {G}auss-{N}ewton for learning over-parameterized models.
\newblock {\em Advances in Neural Information Processing Systems}, 36, 2024.

\bibitem{berahas2016multi}
Albert~S Berahas, Jorge Nocedal, and Martin Tak{\'a}c.
\newblock A multi-batch {L-BFGS} method for machine learning.
\newblock {\em Advances in Neural Information Processing Systems}, 29, 2016.

\bibitem{bollapragada2019exact}
Raghu Bollapragada, Richard~H Byrd, and Jorge Nocedal.
\newblock Exact and inexact subsampled {N}ewton methods for optimization.
\newblock {\em IMA Journal of Numerical Analysis}, 39(2):545--578, 2019.

\bibitem{botev2017practical}
Aleksandar Botev, Hippolyt Ritter, and David Barber.
\newblock Practical {G}auss-{N}ewton optimisation for deep learning.
\newblock In {\em International Conference on Machine Learning}, pages
  557--565. PMLR, 2017.

\bibitem{bottou2018optimization}
L{\'e}on Bottou, Frank~E Curtis, and Jorge Nocedal.
\newblock Optimization methods for large-scale machine learning.
\newblock {\em SIAM review}, 60(2):223--311, 2018.

\bibitem{bradtke1996linear}
Steven~J Bradtke and Andrew~G Barto.
\newblock Linear least-squares algorithms for temporal difference learning.
\newblock {\em Machine learning}, 22(1):33--57, 1996.

\bibitem{bradtke1994adaptive}
Steven~J Bradtke, B~Erik Ydstie, and Andrew~G Barto.
\newblock Adaptive linear quadratic control using policy iteration.
\newblock In {\em Proceedings of 1994 American Control Conference-ACC'94},
  volume~3, pages 3475--3479. IEEE, 1994.

\bibitem{openai_gym}
Greg Brockman, Vicki Cheung, Ludwig Pettersson, Jonas Schneider, John Schulman,
  Jie Tang, and Wojciech Zaremba.
\newblock {OpenAI Gym}, 2016.

\bibitem{brown2020language}
Tom~B. Brown, Benjamin Mann, Nick Ryder, Melanie Subbiah, Jared Kaplan,
  Prafulla Dhariwal, Arvind Neelakantan, Pranav Shyam, Girish Sastry, Amanda
  Askell, Sandhini Agarwal, Ariel Herbert-Voss, Gretchen Krueger, Tom Henighan,
  Rewon Child, Aditya Ramesh, Daniel~M. Ziegler, Jeffrey Wu, Clemens Winter,
  Christopher Hesse, Mark Chen, Eric Sigler, Mateusz Litwin, Scott Gray,
  Benjamin Chess, Jack Clark, Christopher Berner, Sam McCandlish, Alec Radford,
  Ilya Sutskever, and Dario Amodei.
\newblock Language models are few-shot learners.
\newblock 2020.

\bibitem{brust2021nonlinear}
Johannes~J Brust.
\newblock Nonlinear least squares for large-scale machine learning using
  stochastic {J}acobian estimates.
\newblock {\em arXiv preprint arXiv:2107.05598}, 2021.

\bibitem{byrd2016stochastic}
Richard~H Byrd, Samantha~L Hansen, Jorge Nocedal, and Yoram Singer.
\newblock A stochastic quasi-{N}ewton method for large-scale optimization.
\newblock {\em SIAM Journal on Optimization}, 26(2):1008--1031, 2016.

\bibitem{chapelle2011improved}
Olivier Chapelle, Dumitru Erhan, et~al.
\newblock Improved preconditioner for {H}essian free optimization.
\newblock In {\em NIPS Workshop on Deep Learning and Unsupervised Feature
  Learning}, volume 201. Citeseer, 2011.

\bibitem{chen2022demon}
John Chen, Cameron Wolfe, Zhao Li, and Anastasios Kyrillidis.
\newblock Demon: improved neural network training with momentum decay.
\newblock In {\em ICASSP 2022-2022 IEEE International Conference on Acoustics,
  Speech and Signal Processing (ICASSP)}, pages 3958--3962. IEEE, 2022.

\bibitem{curtis2020adaptive}
Frank~E Curtis and Katya Scheinberg.
\newblock Adaptive stochastic optimization: A framework for analyzing
  stochastic optimization algorithms.
\newblock {\em IEEE Signal Processing Magazine}, 37(5):32--42, 2020.

\bibitem{davison1990benchmark}
Edward~J Davison.
\newblock Benchmark problems for control system design.
\newblock {\em Report of the IFAC Theory Committee}, 1990.

\bibitem{defazio2014saga}
Aaron Defazio, Francis Bach, and Simon Lacoste-Julien.
\newblock Saga: A fast incremental gradient method with support for
  non-strongly convex composite objectives.
\newblock {\em Advances in neural information processing systems}, 27, 2014.

\bibitem{defazio2019ineffectiveness}
Aaron Defazio and L{\'e}on Bottou.
\newblock On the ineffectiveness of variance reduced optimization for deep
  learning.
\newblock {\em Advances in Neural Information Processing Systems}, 32, 2019.

\bibitem{doikov2023second}
Nikita Doikov, Martin Jaggi, et~al.
\newblock Second-order optimization with lazy {H}essians.
\newblock In {\em International Conference on Machine Learning}, pages
  8138--8161. PMLR, 2023.

\bibitem{dozat2016incorporating}
Timothy Dozat.
\newblock Incorporating {N}esterov momentum into {A}dam.
\newblock 2016.

\bibitem{dubey2024llama}
Abhimanyu Dubey, Abhinav Jauhri, Abhinav Pandey, Abhishek Kadian, Ahmad
  Al-Dahle, Aiesha Letman, Akhil Mathur, Alan Schelten, Amy Yang, Angela Fan,
  et~al.
\newblock The llama 3 herd of models.
\newblock {\em arXiv preprint arXiv:2407.21783}, 2024.

\bibitem{duchi2011adaptive}
John Duchi, Elad Hazan, and Yoram Singer.
\newblock Adaptive subgradient methods for online learning and stochastic
  optimization.
\newblock {\em Journal of machine learning research}, 12(7), 2011.

\bibitem{duncan1944lxxviii}
William~Jolly Duncan.
\newblock Lxxviii. some devices for the solution of large sets of simultaneous
  linear equations: With an appendix on the reciprocation of partitioned
  matrices.
\newblock {\em The London, Edinburgh, and Dublin Philosophical Magazine and
  Journal of Science}, 35(249):660--670, 1944.

\bibitem{gargiani2020promise}
Matilde Gargiani, Andrea Zanelli, Moritz Diehl, and Frank Hutter.
\newblock On the promise of the stochastic generalized {G}auss-{N}ewton method
  for training {DNNs}.
\newblock {\em arXiv preprint arXiv:2006.02409}, 2020.

\bibitem{ghadimi2013stochastic}
Saeed Ghadimi and Guanghui Lan.
\newblock Stochastic first-and zeroth-order methods for nonconvex stochastic
  programming.
\newblock {\em SIAM journal on optimization}, 23(4):2341--2368, 2013.

\bibitem{gower2020variance}
Robert~M Gower, Mark Schmidt, Francis Bach, and Peter Richt{\'a}rik.
\newblock Variance-reduced methods for machine learning.
\newblock {\em Proceedings of the IEEE}, 108(11):1968--1983, 2020.

\bibitem{grosse2021taylor}
Roger Grosse.
\newblock Taylor approximations.
\newblock {\em Neural Network Training Dynamics. Lecture Notes, University of
  Toronto}, 2021.

\bibitem{guttman1946enlargement}
Louis Guttman.
\newblock Enlargement methods for computing the inverse matrix.
\newblock {\em The annals of mathematical statistics}, pages 336--343, 1946.

\bibitem{ds_superconductivty}
Kam Hamidieh.
\newblock {Superconductivty Data}.
\newblock UCI Machine Learning Repository, 2018.
\newblock {DOI}: https://doi.org/10.24432/C53P47.

\bibitem{hazan2022introduction}
Elad Hazan and Karan Singh.
\newblock Introduction to online nonstochastic control.
\newblock {\em arXiv preprint arXiv:2211.09619}, 2022.

\bibitem{he2016deep}
Kaiming He, Xiangyu Zhang, Shaoqing Ren, and Jian Sun.
\newblock Deep residual learning for image recognition.
\newblock In {\em Proceedings of the IEEE conference on computer vision and
  pattern recognition}, pages 770--778, 2016.

\bibitem{higham2002accuracy}
Nicholas~J Higham.
\newblock {\em Accuracy and stability of numerical algorithms}.
\newblock SIAM, 2002.

\bibitem{hinton2012neural}
Geoffrey Hinton, Nitish Srivastava, and Kevin Swersky.
\newblock Neural networks for machine learning. {L}ecture 6a overview of
  mini-batch gradient descent.
\newblock {\em Cited on}, 14(8):2, 2012.

\bibitem{hong2020stochastic}
Yuxi Hong, Houcine Bergou, Nicolas Doucet, Hao Zhang, Jesse Cranney, Hatem
  Ltaief, Damien Gratadour, Francois Rigaut, and David~E Keyes.
\newblock Stochastic {L}evenberg-{M}arquardt for solving optimization problems
  on hardware accelerators.
\newblock Submitted to IEEE, 2020.

\bibitem{honnibal2020spacy}
Matthew Honnibal, Ines Montani, Sofie Van~Landeghem, Adriane Boyd, et~al.
\newblock {spaCy}: Industrial-strength natural language processing in {P}ython.
\newblock 2020.

\bibitem{huang2021cleanrl}
Shengyi Huang, Rousslan Fernand~Julien Dossa, Chang Ye, and Jeff Braga.
\newblock Cleanrl: High-quality single-file implementations of deep
  reinforcement learning algorithms.
\newblock {\em arXiv preprint arXiv:2111.08819}, 2021.

\bibitem{hung1982multivariable}
YS~Hung and AGJ MacFarlane.
\newblock Multivariable control: A quasiclassical approach, 1982.

\bibitem{johnson2013accelerating}
Rie Johnson and Tong Zhang.
\newblock Accelerating stochastic gradient descent using predictive variance
  reduction.
\newblock {\em Advances in neural information processing systems}, 26, 2013.

\bibitem{kingma2014adam}
Diederik~P Kingma and Jimmy Ba.
\newblock Adam: A method for stochastic optimization.
\newblock {\em arXiv preprint arXiv:1412.6980}, 2014.

\bibitem{kiros2013training}
Ryan Kiros.
\newblock Training neural networks with stochastic {H}essian-free optimization.
\newblock {\em arXiv preprint arXiv:1301.3641}, 2013.

\bibitem{kunstner2019limitations}
Frederik Kunstner, Philipp Hennig, and Lukas Balles.
\newblock Limitations of the empirical {F}isher approximation for natural
  gradient descent.
\newblock {\em Advances in neural information processing systems}, 32, 2019.

\bibitem{leibfritz2006compleib}
F~Leibfritz.
\newblock Compleib: Constrained matrix optimization problem library, 2006.

\bibitem{hf-datasets}
Quentin Lhoest, Albert Villanova~del Moral, Yacine Jernite, Abhishek Thakur,
  Patrick von Platen, Suraj Patil, Julien Chaumond, Mariama Drame, Julien Plu,
  Lewis Tunstall, Joe Davison, Mario {\v{S}}a{\v{s}}ko, Gunjan Chhablani,
  Bhavitvya Malik, Simon Brandeis, Teven Le~Scao, Victor Sanh, Canwen Xu,
  Nicolas Patry, Angelina McMillan-Major, Philipp Schmid, Sylvain Gugger,
  Cl{\'e}ment Delangue, Th{\'e}o Matussi{\`e}re, Lysandre Debut, Stas Bekman,
  Pierric Cistac, Thibault Goehringer, Victor Mustar, Fran{\c{c}}ois Lagunas,
  Alexander Rush, and Thomas Wolf.
\newblock Datasets: A community library for natural language processing.
\newblock In {\em Proceedings of the 2021 Conference on Empirical Methods in
  Natural Language Processing: System Demonstrations}, pages 175--184, Online
  and Punta Cana, Dominican Republic, November 2021. Association for
  Computational Linguistics.

\bibitem{liu2022quasi}
Chengchang Liu and Luo Luo.
\newblock Quasi-{N}ewton methods for saddle point problems.
\newblock {\em Advances in Neural Information Processing Systems},
  35:3975--3987, 2022.

\bibitem{liu2023sophia}
Hong Liu, Zhiyuan Li, David Hall, Percy Liang, and Tengyu Ma.
\newblock Sophia: A scalable stochastic second-order optimizer for language
  model pre-training.
\newblock {\em arXiv preprint arXiv:2305.14342}, 2023.

\bibitem{liu2021activated}
Mei Liu, Liangming Chen, Xiaohao Du, Long Jin, and Mingsheng Shang.
\newblock Activated gradients for deep neural networks.
\newblock {\em IEEE Transactions on Neural Networks and Learning Systems},
  34(4):2156--2168, 2021.

\bibitem{loshchilov2017decoupled}
Ilya Loshchilov and Frank Hutter.
\newblock Decoupled weight decay regularization.
\newblock {\em arXiv preprint arXiv:1711.05101}, 2017.

\bibitem{lucas2018aggregated}
James Lucas, Shengyang Sun, Richard Zemel, and Roger Grosse.
\newblock Aggregated momentum: Stability through passive damping.
\newblock {\em arXiv preprint arXiv:1804.00325}, 2018.

\bibitem{ds_imdb}
Andrew~L. Maas, Raymond~E. Daly, Peter~T. Pham, Dan Huang, Andrew~Y. Ng, and
  Christopher Potts.
\newblock Learning word vectors for sentiment analysis.
\newblock In {\em Proceedings of the 49th Annual Meeting of the Association for
  Computational Linguistics: Human Language Technologies}, pages 142--150,
  Portland, Oregon, USA, June 2011. Association for Computational Linguistics.

\bibitem{martens2010deep}
James Martens et~al.
\newblock Deep learning via {H}essian-free optimization.
\newblock In {\em ICML}, volume~27, pages 735--742, 2010.

\bibitem{martens2011learning}
James Martens and Ilya Sutskever.
\newblock Learning recurrent neural networks with {H}essian-free optimization.
\newblock In {\em Proceedings of the 28th international conference on machine
  learning (ICML-11)}, pages 1033--1040, 2011.

\bibitem{mnih2013playing}
Volodymyr Mnih, Koray Kavukcuoglu, David Silver, Alex Graves, Ioannis
  Antonoglou, Daan Wierstra, and Martin Riedmiller.
\newblock Playing {A}tari with deep reinforcement learning.
\newblock {\em arXiv preprint arXiv:1312.5602}, 2013.

\bibitem{pml1Book}
Kevin~P. Murphy.
\newblock {\em Probabilistic Machine Learning: An introduction}.
\newblock MIT Press, 2022.

\bibitem{naumov2019deep}
Maxim Naumov, Dheevatsa Mudigere, Hao-Jun~Michael Shi, Jianyu Huang, Narayanan
  Sundaraman, Jongsoo Park, Xiaodong Wang, Udit Gupta, Carole-Jean Wu,
  Alisson~G Azzolini, et~al.
\newblock Deep learning recommendation model for personalization and
  recommendation systems.
\newblock {\em arXiv preprint arXiv:1906.00091}, 2019.

\bibitem{nesterov1983method}
Yurii Nesterov.
\newblock A method of solving a convex programming problem with convergence
  rate ${O}(1/k^2)$.
\newblock {\em Doklady Akademii Nauk SSSR}, 269(3):543, 1983.

\bibitem{nguyen2017sarah}
Lam~M Nguyen, Jie Liu, Katya Scheinberg, and Martin Tak{\'a}{\v{c}}.
\newblock Sarah: A novel method for machine learning problems using stochastic
  recursive gradient.
\newblock In {\em International conference on machine learning}, pages
  2613--2621. PMLR, 2017.

\bibitem{nocedal2006numerical}
J.~Nocedal and S.~Wright.
\newblock {\em Numerical Optimization}.
\newblock Springer Series in Operations Research and Financial Engineering.
  Springer New York, 2006.

\bibitem{ozyildirim2021levenberg}
Buse~Melis Ozyildirim and Mariam Kiran.
\newblock Levenberg--{M}arquardt multi-classification using hinge loss
  function.
\newblock {\em Neural Networks}, 143:564--571, 2021.

\bibitem{ds_california_housing}
R~Kelley Pace and Ronald Barry.
\newblock Sparse spatial autoregressions.
\newblock {\em Statistics \& Probability Letters}, 33(3):291--297, 1997.

\bibitem{papyan2018full}
Vardan Papyan.
\newblock The full spectrum of deep net {H}essians at scale: Dynamics with
  sample size.
\newblock {\em arXiv preprint arXiv:1811.07062}, 2018.

\bibitem{paquette2020stochastic}
Courtney Paquette and Katya Scheinberg.
\newblock A stochastic line search method with expected complexity analysis.
\newblock {\em SIAM Journal on Optimization}, 30(1):349--376, 2020.

\bibitem{scikit-learn}
F.~Pedregosa, G.~Varoquaux, A.~Gramfort, V.~Michel, B.~Thirion, O.~Grisel,
  M.~Blondel, P.~Prettenhofer, R.~Weiss, V.~Dubourg, J.~Vanderplas, A.~Passos,
  D.~Cournapeau, M.~Brucher, M.~Perrot, and E.~Duchesnay.
\newblock Scikit-learn: Machine learning in {P}ython.
\newblock {\em Journal of Machine Learning Research}, 12:2825--2830, 2011.

\bibitem{pooladzandi2022improving}
Omead Pooladzandi and Yiming Zhou.
\newblock Improving {L}evenberg-{M}arquardt algorithm for neural networks.
\newblock {\em arXiv preprint arXiv:2212.08769}, 2022.

\bibitem{ren2019efficient}
Yi~Ren and Donald Goldfarb.
\newblock Efficient subsampled gauss-newton and natural gradient methods for
  training neural networks.
\newblock {\em arXiv preprint arXiv:1906.02353}, 2019.

\bibitem{robbins1951stochastic}
Herbert Robbins and Sutton Monro.
\newblock A stochastic approximation method.
\newblock {\em The annals of mathematical statistics}, pages 400--407, 1951.

\bibitem{sagun2017empirical}
Levent Sagun, Utku Evci, V~Ugur Guney, Yann Dauphin, and Leon Bottou.
\newblock Empirical analysis of the hessian of over-parametrized neural
  networks.
\newblock {\em arXiv preprint arXiv:1706.04454}, 2017.

\bibitem{sankar2021deeper}
Adepu~Ravi Sankar, Yash Khasbage, Rahul Vigneswaran, and Vineeth~N
  Balasubramanian.
\newblock A deeper look at the {H}essian eigenspectrum of deep neural networks
  and its applications to regularization.
\newblock In {\em Proceedings of the AAAI Conference on Artificial
  Intelligence}, volume~35, pages 9481--9488, 2021.

\bibitem{schraudolph2002fast}
Nicol~N Schraudolph.
\newblock Fast curvature matrix-vector products for second-order gradient
  descent.
\newblock {\em Neural computation}, 14(7):1723--1738, 2002.

\bibitem{schraudolph2007stochastic}
Nicol~N Schraudolph, Jin Yu, and Simon G{\"u}nter.
\newblock A stochastic quasi-newton method for online convex optimization.
\newblock In {\em Artificial intelligence and statistics}, pages 436--443.
  PMLR, 2007.

\bibitem{sun2019survey}
Shiliang Sun, Zehui Cao, Han Zhu, and Jing Zhao.
\newblock A survey of optimization methods from a machine learning perspective.
\newblock {\em IEEE transactions on cybernetics}, 50(8):3668--3681, 2019.

\bibitem{sutskever2013importance}
Ilya Sutskever, James Martens, George Dahl, and Geoffrey Hinton.
\newblock On the importance of initialization and momentum in deep learning.
\newblock In {\em International conference on machine learning}, pages
  1139--1147. PMLR, 2013.

\bibitem{tfds}
TensorFlow.
\newblock Tensorflow datasets, a collection of ready-to-use datasets.
\newblock \url{https://www.tensorflow.org/datasets}.

\bibitem{touvron2023llama}
Hugo Touvron, Thibaut Lavril, Gautier Izacard, Xavier Martinet, Marie-Anne
  Lachaux, Timoth{\'e}e Lacroix, Baptiste Rozi{\`e}re, Naman Goyal, Eric
  Hambro, Faisal Azhar, et~al.
\newblock Llama: Open and efficient foundation language models.
\newblock {\em arXiv preprint arXiv:2302.13971}, 2023.

\bibitem{tran2020stochastic}
Quoc Tran-Dinh, Nhan Pham, and Lam Nguyen.
\newblock Stochastic {G}auss-{N}ewton algorithms for nonconvex compositional
  optimization.
\newblock In {\em International Conference on Machine Learning}, pages
  9572--9582. PMLR, 2020.

\bibitem{vaswani2017attention}
Ashish Vaswani, Noam Shazeer, Niki Parmar, Jakob Uszkoreit, Llion Jones,
  Aidan~N Gomez, {\L}ukasz Kaiser, and Illia Polosukhin.
\newblock Attention is all you need.
\newblock {\em Advances in neural information processing systems}, 30, 2017.

\bibitem{vaswani2019painless}
Sharan Vaswani, Aaron Mishkin, Issam Laradji, Mark Schmidt, Gauthier Gidel, and
  Simon Lacoste-Julien.
\newblock Painless stochastic gradient: Interpolation, line-search, and
  convergence rates.
\newblock {\em Advances in neural information processing systems}, 32, 2019.

\bibitem{vinyals2012krylov}
Oriol Vinyals and Daniel Povey.
\newblock Krylov subspace descent for deep learning.
\newblock In {\em Artificial intelligence and statistics}, pages 1261--1268.
  PMLR, 2012.

\bibitem{ds_diamonds}
Hadley Wickham.
\newblock {\em ggplot2: Elegant Graphics for Data Analysis}.
\newblock Springer-Verlag New York, 2016.

\bibitem{wills2021stochastic}
Adrian~G Wills and Thomas~B Sch{\"o}n.
\newblock Stochastic quasi-{N}ewton with line-search regularisation.
\newblock {\em Automatica}, 127:109503, 2021.

\bibitem{grok1}
{xai-org}.
\newblock Grok-1.
\newblock \url{https://github.com/xai-org/grok-1}, 2024.
\newblock GitHub repository.

\bibitem{xu2020second}
Peng Xu, Fred Roosta, and Michael~W Mahoney.
\newblock Second-order optimization for non-convex machine learning: An
  empirical study.
\newblock In {\em Proceedings of the 2020 SIAM International Conference on Data
  Mining}, pages 199--207. SIAM, 2020.

\bibitem{yao2021adahessian}
Zhewei Yao, Amir Gholami, Sheng Shen, Mustafa Mustafa, Kurt Keutzer, and
  Michael Mahoney.
\newblock Adahessian: An adaptive second order optimizer for machine learning.
\newblock In {\em proceedings of the AAAI conference on artificial
  intelligence}, volume~35, pages 10665--10673, 2021.

\bibitem{young19minatar}
Kenny {Young} and Tian {Tian}.
\newblock Minatar: An {A}tari-inspired testbed for thorough and reproducible
  reinforcement learning experiments.
\newblock {\em arXiv preprint arXiv:1903.03176}, 2019.

\bibitem{zaheer2018adaptive}
Manzil Zaheer, Sashank Reddi, Devendra Sachan, Satyen Kale, and Sanjiv Kumar.
\newblock Adaptive methods for nonconvex optimization.
\newblock {\em Advances in neural information processing systems}, 31, 2018.

\bibitem{zeiler2012adadelta}
Matthew~D Zeiler.
\newblock Adadelta: an adaptive learning rate method.
\newblock {\em arXiv preprint arXiv:1212.5701}, 2012.

\bibitem{zhang2023optimal}
Fangzhao Zhang and Mert Pilanci.
\newblock Optimal shrinkage for distributed second-order optimization.
\newblock In {\em International Conference on Machine Learning}, pages
  41523--41549. PMLR, 2023.

\end{thebibliography}


\newpage

\appendix

\section{Proofs and Derivations}\label{apx:derivations}

\subsection{Generalized Gauss-Newton Hessian Approximation}\label{apx:ggn}

\paragraph{Claim} The 
generalized Gauss-Newton Hessian approximation
scheme for the batch loss~\eqref{eq:batch_loss} is
$\mathbf{H}^{\mathrm{GN}}=\frac{1}{b}\mathbf{J}^{\top}\mathbf{Q}\mathbf{J}$.

\begin{proof}
The derivative of the generic batch loss function~\eqref{eq:batch_loss}
reads
\begin{align*}
    \frac{\partial}{\partial\mathbf{w}}\mathcal{L}_{b}=\frac{\partial}{\partial\mathbf{w}}\left[\frac{1}{b}\sum_{i=1}^{b}\ell\left(\mathbf{y}_{i},\Phi(\mathbf{x}_{i};\mathbf{w})\right)\right]
\end{align*}
Employing the chain rule $\frac{\partial\ell}{\partial\mathbf{w}}=\frac{\partial\ell}{\partial \Phi }\frac{\partial \Phi }{\partial\mathbf{w}}$
we have
\begin{align*}
\frac{\partial}{\partial\mathbf{w}}\mathcal{L}_{b}=\frac{1}{b}\sum_{i=1}^{b}\frac{\partial\ell\left(\mathbf{y}_{i},\Phi(\mathbf{x}_{i};\mathbf{w})\right)}{\partial\Phi}\frac{\partial\Phi(\mathbf{x}_{i};\mathbf{w})}{\partial\mathbf{w}}=\frac{1}{b}\sum_{i=1}^{b}\mathbf{L}_{i}\mathbf{J}_{\Phi_{i}},
\end{align*}
where
$\mathbf{L}_{i}=\frac{\partial\ell\left(\mathbf{y}_{i},\Phi(\mathbf{x}_{i};\mathbf{w})\right)}{\partial \Phi }\in\rr^{1\times c}$ 
and $\mathbf{J}_{\Phi_i}=\frac{\partial \Phi(\mathbf{x}_{i};\mathbf{w})}{\partial\mathbf{w}}\in\rr^{c\times d}$.
We obtain the Hessian by differentiating the gradient
\begin{align*}
\mathbf{H}=\frac{\partial}{\partial\mathbf{w}}\left[\frac{1}{b}\sum_{i=1}^{b}\mathbf{L}_{i}\mathbf{J}_{\Phi_{i}}\right].
\end{align*}
Using the product rule we arrive at
\begin{align*}
\mathbf{H}=\frac{1}{b}\sum_{i=1}^{b}\left(\frac{\partial}{\partial\mathbf{w}}\left[\mathbf{L}_{i}\right]\mathbf{J}_{\Phi_{i}}+\sum_{k=1}^{c}\mathbf{L}_{i,k}\frac{\partial}{\partial\mathbf{w}}\left[\mathbf{J}_{\Phi_{i,k}}\right]\right),
\end{align*}
where 
$\mathbf{L}_{i,k} \in \rr$ 
is the $k$-th element of the derivative of the 
loss with respect to $\Phi$,
and $\mathbf{J}_{\Phi_{i,k}} \in \rr^{1 \times d}$ is the $k$-th row
the Jacobian $\mathbf{J}_{\Phi_{i}}$.
We obtain the Gauss-Newton Hessian approximation 
by neglecting the second-term 
of $\mathbf{H}$~\cite{papyan2018full, sagun2017empirical}, 
such that
\begin{align*}
    \mathbf{H}^{\mathrm{GN}}=\frac{1}{b}\sum_{i=1}^{b}\frac{\partial}{\partial\mathbf{w}}\left[\mathbf{L}_{i}\right]\mathbf{J}_{\Phi_{i}}.
\end{align*}
Given that
\begin{align*}
    \frac{\partial}{\partial\mathbf{w}}\left[\mathbf{L}_{i}\right]=\mathbf{J}_{\Phi_i}^{\top}\frac{\partial^{2}\ell\left(\mathbf{y}_{i},\Phi(\mathbf{x}_{i};\mathbf{w})\right)}{\partial \Phi ^{2}}=\mathbf{J}_{\Phi_i}^{\top}\mathbf{Q}_{\ell_i},
\end{align*}
where $\mathbf{Q}_{\ell_i}=\frac{\partial^{2}\ell(\mathbf{y}_{i},\Phi(\mathbf{x}_{i};\mathbf{w}))}{\partial \Phi ^{2}}\in\rr^{c\times c}$ is the second derivative of the loss with respect to the function's output,
we have
\begin{align*}
\mathbf{H}^{\mathrm{GN}}=\frac{1}{b}\sum_{i=1}^{b}\mathbf{J}_{\Phi_{i}}^{\top}\mathbf{Q}_{\ell_{i}}\mathbf{J}_{\Phi_{i}}.
\end{align*}
Or using the compact notation
\begin{align*}
    \mathbf{H}^{\mathrm{GN}}=\frac{1}{b}\mathbf{J}^{\top}\mathbf{Q}\mathbf{J},
\end{align*}
where we vertically stack individual Jacobians $\mathbf{J}_{\Phi_i}$
for each sample in the batch $\mathcal{B}$ to form
$\mathbf{J} \in \rr^{bc \times d}$
and
form a block diagonal matrix
$\mathbf{Q} = \text{blkdiag}(\mathbf{Q}_{\ell_1}, \mathbf{Q}_{\ell_2}, \ldots, \mathbf{Q}_{\ell_b}) \in \rr^{bc \times bc}$.
\end{proof}

\subsection{Gauss-Newton Hessian of the MSE Loss Function}\label{apx:mse}

\paragraph{Claim}
The gradient of the batch MSE loss is 
$\mathbf{g}=\frac{1}{b}\mathbf{J}^{\top}\mathbf{r}$.

\begin{proof}
We define the residual vector $\mathbf{r}$ as 
$\mathbf{r}:=\left[\begin{array}{ccc}
\Phi(\mathbf{x}_{1};\mathbf{w})-\mathbf{y}_{1} & \dots & \Phi(\mathbf{x}_{b};\mathbf{w})-\mathbf{y}_{b}\end{array}\right]^{\top}$. 
Recall that the stacked Jacobians of the neural network for 
are denoted by $\mathbf{J}$,
with $\mathbf{J}\in\rr^{b\times d}$ for the regression task. 
The batch loss~\eqref{eq:batch_loss} for MSE is
\begin{align*}
    \mathcal{L}_{b}\left(\mathbf{w}\right)=\frac{1}{2b}\sum_{i=1}^{b}\left(\Phi(\mathbf{x}_{i};\mathbf{w})-\mathbf{y}_{i}\right)^{2}=\frac{1}{2b}\mathbf{r}^{\top}\mathbf{r}.
\end{align*}
So that
\begin{align*}
\frac{\partial}{\partial\mathbf{w}}\left[\mathcal{L}_{b}\right]=\frac{\partial\mathcal{L}_{b}}{\partial\mathbf{r}}\frac{\partial\mathbf{r}}{\partial\Phi}\frac{\partial\Phi}{\partial\mathbf{w}}=\frac{1}{b}\mathbf{r}^{\top}\mathbf{J}.
\end{align*}
Since we define the gradient to be a column vector
\begin{align*}
    \mathbf{g}=\left(\frac{\partial\mathcal{L}_{b}}{\partial\mathbf{w}}\right)^{\top}=\left(\frac{1}{b}\mathbf{r}^{\top}\mathbf{J}\right)^{\top}=\frac{1}{b}\mathbf{J}^{\top}\mathbf{r}.
\end{align*}
\end{proof}

\paragraph{Claim}
The Hessian of the batch MSE loss is $\mathbf{H}^{\text{GN}}=\frac{1}{b}\mathbf{J}^{\top}\mathbf{J}$.

\begin{proof}
We obtain the Hessian by differentiating the gradient of the loss
\begin{align*}
    \mathbf{H}=\frac{\partial}{\partial\mathbf{w}}\mathbf{g}=\frac{1}{b}\frac{\partial}{\partial\mathbf{w}}\left(\mathbf{J}^{\top}\mathbf{r}\right).
\end{align*}
Using the product rule of vector calculus
\begin{align*}
    \mathbf{H} = \frac{1}{b}\left(\mathbf{J}^{\top}\frac{\partial}{\partial\mathbf{w}}\left[\mathbf{r}\right] 
    + \left(\frac{\partial}{\partial\mathbf{w}}\left[\mathbf{J}^{\top}\right]\right)\mathbf{r}\right) = \frac{1}{b}\left(\mathbf{J}^{\top}\mathbf{J} 
    + \sum_{i=1}^{b}\frac{\partial^{2}}{\partial\mathbf{w}^{2}}\left[\Phi(\mathbf{x}_{i};\mathbf{w})\right]\mathbf{r}_{i}\right).
\end{align*}
where $\mathbf{r}_{i} \in \rr$ is the $i$-th element of $\mathbf{r}$.

Neglecting the second term we obtain the Gauss-Newton approximation of the Hessian
\begin{align*}
\mathbf{H}^{\text{GN}}=\frac{1}{b}\mathbf{J}^{\top}\mathbf{J}.
\end{align*}
\end{proof}

\subsection{Gauss-Newton Hessian of the Multi-class Cross-entropy Loss Function}\label{apx:ce}

\paragraph{Claim}
The gradient of the batch multi-class cross-entropy loss is 
$\mathbf{g}=\frac{1}{b}\mathbf{J}^{\top}\mathbf{r}$.

\begin{proof}
We recall from~\ref{apx:ggn} that
the partial derivative of the generic loss function is
\begin{align*}
\frac{\partial}{\partial\mathbf{w}}\mathcal{L}_{b}=\frac{1}{b}\sum_{i=1}^{b}\frac{\partial\ell\left(\mathbf{y}_{i},\Phi(\mathbf{x}_{i};\mathbf{w})\right)}{\partial\Phi}\mathbf{J}_{\Phi_{i}}.
\end{align*}
Assuming that $\mathbf{J}_{\Phi_{i}}$ is calculated during the backpropagation stage,
we examine the term $\frac{\partial\ell\left(\mathbf{y}_{i},\Phi(\mathbf{x}_{i};\mathbf{w})\right)}{\partial\Phi}$.
The cross-entropy loss function~\eqref{eq:loss_ce}
is defined as
\begin{align*}
\ell\left(\mathbf{y}_{i},\Phi(\mathbf{x}_{i};\mathbf{w})\right)\coloneqq-\mathbf{y}_{i}^{\top}\log\left(\sigma\left(\Phi(\mathbf{x}_{i};\mathbf{w})\right)\right),
\end{align*}
where
$\sigma$ is the softmax function defined as
\begin{align*}
    \sigma(\mathbf{z}_{i,k}) = \frac{e^{\mathbf{z}_{i,k}}}{\sum_{j=1}^{c} e^{\mathbf{z}_{i,j}}},
\end{align*}
with a $c$-dimensional vector of prediction scores (logits)
$\mathbf{z}_{i}=\Phi(\mathbf{x}_{i};\mathbf{w})$
and 
a $c$-dimensional vector of probabilities
$\mathbf{p}_{i}=\sigma(\mathbf{z}_{i})$.
Applying the chain rule and using the shorthand notation we obtain 
\begin{align*}
\frac{\partial\ell\left(\mathbf{y}_{i},\Phi(\mathbf{x}_{i};\mathbf{w})\right)}{\partial\Phi}=\frac{\partial\ell}{\partial\mathbf{p}_{i}}\frac{\partial\mathbf{p}_{i}}{\partial\mathbf{z}_{i}}\frac{\partial\mathbf{z}_{i}}{\partial\Phi}.
\end{align*}
Differentiating the loss with respect to the softmax function yields
\begin{align*}
\frac{\partial\ell}{\partial\mathbf{p}_{i}}=\left[-\frac{\mathbf{y}_{i,1}}{\mathbf{p}_{i,1}},-\frac{\mathbf{y}_{i,2}}{\mathbf{p}_{i,2}},\ldots,-\frac{\mathbf{y}_{i,c}}{\mathbf{p}_{i,c}}\right].
\end{align*}
The derivative of the $k$-th element of the softmax function 
wrt the $l$-th input $\mathbf{z}_{i,l}$ is split in two cases
\begin{align*}
\frac{\partial\mathbf{p}_{i,k}}{\partial\mathbf{z}_{i,l}}=\begin{cases}
\mathbf{p}_{i,k}(1-\mathbf{p}_{i,k}), & \text{if }k=l\\
-\mathbf{p}_{i,k}\mathbf{p}_{i,l}, & \text{if }k\neq l.
\end{cases}
\end{align*}
So that the derivative 
$\frac{\partial\mathbf{p}_{i}}{\partial\mathbf{z}_{i}} \in \rr^{c \times c}$
can be structured as
\begin{align*}
\frac{\partial\mathbf{p}_{i}}{\partial\mathbf{z}_{i}}=\text{diag}(\mathbf{p}_{i})-\mathbf{p}_{i}\mathbf{p}_{i}^{\top}.    
\end{align*}
Combining $\frac{\partial\ell}{\partial\mathbf{p}_{i}}$ and
$\frac{\partial\mathbf{p}_{i}}{\partial\mathbf{z}_{i}}$, 
we obtain
\begin{align*}
\frac{\partial\ell}{\partial\mathbf{z}_{i,k}}=-\sum_{l=1}^{c}\left(\frac{\mathbf{y}_{i,l}}{\mathbf{p}_{i,l}}\right)\begin{cases}
\mathbf{p}_{i,k}(1-\mathbf{p}_{i,k}), & \text{if }k=l\\
-\mathbf{p}_{i,k}\mathbf{p}_{i,l}, & \text{if }k\neq l,
\end{cases}    
\end{align*}
Which simplifies to
\begin{align*}
\frac{\partial\ell}{\partial\mathbf{z}_{i,k}}=-\mathbf{y}_{i,k}+\mathbf{p}_{i,k}=\mathbf{p}_{i,k}-\mathbf{y}_{i,k}.    
\end{align*}
Since 
$\frac{\partial\mathbf{z}_{i}}{\partial\Phi}$ is the identity,
we have
\begin{align*}
    \frac{\partial\ell\left(\mathbf{y}_{i},\Phi(\mathbf{x}_{i};\mathbf{w})\right)}{\partial\Phi}=\mathbf{p}_{i}-\mathbf{y}_{i}.
\end{align*}

The same way as for the MSE loss,
we define (pseudo-)residuals $\mathbf{r}_i\in\rr^{c}$
as a column vector of 
probabilities minus the targets, i.e.,
$\left( \mathbf{p}_{i}-\mathbf{y}_{i} \right)^{\top}$.
Substituting back into the loss derivative
\begin{align*}
    \frac{\partial\mathcal{L}_{b}}{\partial\mathbf{w}}=\frac{1}{b}\sum_{i=1}^{b}\mathbf{r}^{\top}_{i}\mathbf{J}_{\Phi_i}.
\end{align*}
Since we define the gradient to be a column vector
\begin{align*}
    \mathbf{g}=\left(\frac{\partial\mathcal{L}_{b}}{\partial\mathbf{w}}\right)^{\top}=\left(\frac{1}{b}\sum_{i=1}^{b} \mathbf{r}^{\top}_i \mathbf{J}_{\Phi_i}\right)^{\top}=\frac{1}{b}\sum_{i=1}^{b}\mathbf{J}_{\Phi_i}^{\top}\mathbf{r}_{i}.
\end{align*}
Or in shorthand notation
\begin{align*}
    \mathbf{g}=\frac{1}{b}\mathbf{J}^{\top}\mathbf{r}
\end{align*}
where for each sample $i$ in the batch $\mathcal{B}$ 
we vertically stack 
Jacobians $\mathbf{J}_{\Phi_i}$
as well as residuals $\mathbf{r}_i$
to form
$\mathbf{J} \in \rr^{bc \times d}$
and 
$\mathbf{r} \in \rr^{bc}$.
\end{proof}

\paragraph{Claim}
The Hessian of the batch multi-class cross-entropy loss is 
$\mathbf{H}^{\text{GN}}=\frac{1}{b}\mathbf{J}^{\top}\mathbf{Q}\mathbf{J}$.

\begin{proof}
Recall that the GN Hessian 
for the generalized case (\ref{apx:ggn}) is  
\begin{align}
    \mathbf{H}^{\mathrm{GN}}=\frac{1}{b}\sum_{i=1}^{b}\mathbf{J}_{\Phi_i}^{\top}\mathbf{Q}_{\ell_i}\mathbf{J}_{\Phi_i}.
\end{align}
The second derivative of the loss with respect to the function's output
$\mathbf{Q}_{\ell_i}=\frac{\partial^{2}\ell(\mathbf{y}_{i},\Phi(\mathbf{x}_{i};\mathbf{w}))}{\partial \Phi ^{2}}\in \rr^{c\times c}$
for CE loss is
\begin{align*}
    \mathbf{Q}_{\ell_i}=\frac{\partial}{\partial \Phi }\left[\mathbf{p}_{i}-\mathbf{y}_{i}\right].
\end{align*}
Using the gradient of the softmax function derived earlier,
we form a symmetric matrix 
$\mathbf{Q}_{\ell_i}$ such that 
\begin{align*}
\mathbf{Q}_{\ell_{i}}=\left[\begin{array}{cccc}
\mathbf{p}_{i1}(1-\mathbf{p}_{i1}) & -\mathbf{p}_{i1}\mathbf{p}_{i2} & \ldots & -\mathbf{p}_{i1}\mathbf{p}_{ic}\\
-\mathbf{p}_{i2}\mathbf{p}_{i1} & \mathbf{p}_{i2}(1-\mathbf{p}_{i2}) & \ldots & -\mathbf{p}_{i2}\mathbf{p}_{ic}\\
\vdots & \vdots & \ddots & \vdots\\
-\mathbf{p}_{ic}\mathbf{p}_{i1} & -\mathbf{p}_{ic}\mathbf{p}_{i2} & \ldots & \mathbf{p}_{ic}(1-\mathbf{p}_{ic})
\end{array}\right],
\end{align*}
Or in compact notation
\begin{align*}
    \mathbf{H}^{\text{GN}}=\frac{1}{b}\mathbf{J}^{\top}\mathbf{Q}\mathbf{J},
\end{align*}
Where
$\mathbf{Q}$ is a block diagonal matrix
\begin{align*}
    \ensuremath{\mathbf{Q}}=\left[\begin{array}{cccc}
        \mathbf{Q}_{\ell_1} & 0 & 0 & 0\\
        0 & \mathbf{Q}_{\ell_2} & 0 & 0\\
        0 & 0 & \ddots & 0\\
        0 & 0 & 0 & \mathbf{Q}_{\ell_b}
    \end{array}\right]
\end{align*}
And 
$\mathbf{J} \in \rr^{bc \times d}$
is vertically stacked 
Jacobians $\mathbf{J}_{\Phi_i}$.
    
\end{proof}

\section{Algorithms}\label{apx:algo}

\begin{algorithm}[H]\label{alg:qr_slm}
	\caption{Calculate direction using QR factorization (MSE loss)}
	\begin{algorithmic}[1]
		\STATE {\bfseries Input:} (pseudo-)residuals $\mathbf{r}$,
        stacked Jacobians of the model $\mathbf{J}$,
        regularizer $\lambda$.
    
        \STATE Factorize $\mathbf{J}^{\top}$ with economy sized QR: $\mathbf{Q},\mathbf{R}\leftarrow\text{qr}\left(\mathbf{J}^{\top}\right)$

        \STATE Factorize: $\mathbf{\tilde{Q}},\mathbf{\tilde{R}}\leftarrow\text{qr}\left(\mathbf{R}\mathbf{R}^{\top}+\lambda\mathbf{I}\right)$
        
        \STATE Solve the linear system for $\delta$: $\mathbf{\tilde{R}}\delta=\mathbf{\tilde{Q}^{\top}}\mathbf{R}\mathbf{r}$

        \STATE Calculate $\mathbf{d}^{\mathrm{LM}}=-\mathbf{Q} \delta$

		\STATE \textbf{Return} $\mathbf{d}^{\mathrm{LM}}$
	\end{algorithmic}
\end{algorithm}

\begin{algorithm}[H]
    \caption{Armijo Line Search}
    \label{alg:armijo_ls}
    \begin{algorithmic}[1]
        \STATE {\bfseries Input:} direction $\mathbf{d}_t$,
        hyper-parameters $\alpha^{\text{max}}$,
        $\kappa$, $c^{\text{up}}$, $c^{\text{down}}$.
        
        \STATE Initialize $\alpha_t \leftarrow \min\left\{ \alpha^{\text{max}},\alpha_{t-1}c^{\text{up}}\right\} $ 
        
        \WHILE{$\mathcal{L}\left(\mathbf{w}_{t+1}\right) > \mathcal{L}\left(\mathbf{w}_{t}\right)+ \kappa\alpha_{t}\nabla\mathcal{L}\left(\mathbf{w}_{t}\right)^{\top}\mathbf{d}_{t}$}
        \STATE Update $\alpha_{t}\leftarrow\alpha_{t}c^{\text{down}}$ 
        
        \STATE Update $\mathbf{w}_{t+1}\leftarrow\mathbf{w}_{t}+\alpha_{t}\mathbf{d}_{t}$ 
        
        \ENDWHILE
        
        \STATE \textbf{Return} $\alpha_t$
    \end{algorithmic}
\end{algorithm}

\begin{algorithm}[H]\label{alg:adaptive_reg}
    \caption{Adaptive regularization~\cite{kiros2013training}}
    \label{alg:adaptive_lambda}
    \begin{algorithmic}[1]
        \STATE {\bfseries Input:} 
        batch $\mathcal{B}_t$, current weights $\mathbf{w}_{t}$, updated weights $\mathbf{w}_{t+1}$.
        
        \STATE Calculate $\rho$ according to~\eqref{eq:rho_lm}
        
        
        \IF{$\rho < 0.25$}
        \STATE $\lambda_{t+1} \leftarrow 1.01 \lambda_{t}$
        \ELSIF{$\rho > 0.75$}
        \STATE $\lambda_{t+1} \leftarrow 0.99 \lambda_{t}$
        \ELSE
        \STATE $\lambda_{t+1} \leftarrow \lambda_{t}$
        \ENDIF
        
        \STATE \textbf{Return} $\lambda_{t+1}$
    \end{algorithmic}
\end{algorithm}

\section{Experiment Details}\label{apx:experiments}

\subsection{Supervised Learning}\label{apx:sl}

\begin{table*}[t]

\centering
\caption{Optimal Hyper-parameters for Supervised Learning Tasks}
\label{tab:hps_sl}
\resizebox{\textwidth}{!}{%
\begin{tabular}{@{}lcccccc@{}}
\toprule
\textbf{Optimizer} & \textbf{Learning Rate} & \textbf{Regularizer} & \textbf{Momentum} & \textbf{Line Search} & \textbf{\#CG Iterations} \\ \midrule
\multicolumn{6}{c}{\textbf{California Housing}} \\
SGD       & 0.03    & -    & -    & -     & -    \\
Adam      & 0.001   & -    & -    & -     & -    \\
GAF       & 0.08    & -    & 0.9  & -     & -    \\
EGN       & 0.4     & 1.0  & 0.9  & False & -    \\
SGN       & 0.2     & 1.0  & -    & -     & 5    \\
SQN       & 0.3     & -    & 0.0  & -     & -    \\
\multicolumn{6}{c}{\textbf{Superconduct}} \\
SGD       & 0.0003  & -    & -    & -     & -    \\
Adam      & 0.01    & -    & -    & -     & -    \\
GAF       & 0.007   & -    & 0.9  & -     & -    \\
EGN       & 0.05    & 1.0  & 0.0  & False & -    \\
SGN       & 0.1     & 1.0  & -    & -     & 10   \\
SQN       & 0.07    & -    & 0.0  & -     & -    \\
\multicolumn{6}{c}{\textbf{Diamonds}} \\
SGD       & 2e-8    & -    & -    & -     & -    \\
Adam      & 0.0005  & -    & -    & -     & -    \\
GAF       & 0.001   & -    & 0.0  & -     & -    \\
EGN       & 0.0005  & 1.0  & 0.0  & False & -    \\
SGN       & 0.001   & 1.0  & -    & -     & 5    \\
SQN       & 0.004   & -    & 0.0  & -     & -    \\
\multicolumn{6}{c}{\textbf{IMDB Reviews}} \\
SGD       & 0.005   & -    & -    & -     & -    \\
Adam      & 0.01    & -    & -    & -     & -    \\
GAF       & 0.003   & -    & 0.9  & -     & -    \\
EGN       & 0.01    & 1.0  & 0.0  & False & -    \\
SGN       & 0.05    & 1.0  & -    & -     & 5    \\
SQN       & 0.02    & -    & 0.0  & -     & -    \\
\bottomrule
\end{tabular}%
}

\end{table*}

\textbf{California Housing}~\cite{ds_california_housing}, 
a part of the scikit-learn~\cite{scikit-learn} datasets package,
consists of $20640$
samples with $8$ numerical features $\mathbf{x}$ 
encoding relevant information, e.g., location, median income, etc.;
and a real-valued target $\mathbf{y}$ representing   
the median house value in California as 
recorded by the 1990 U.S. Census. 

\textbf{Superconductivity}~\cite{ds_superconductivty} is a dataset
of HuggingFace Datasets~\cite{hf-datasets} 
that contains $21263$ instances of $79$ numerical attributes 
(features $\mathbf{x}$) and critical temperatures (target $\mathbf{y}$) 
of superconductors.

\textbf{Diamonds}~\cite{ds_diamonds} is a TFDS~\cite{tfds} 
dataset containing
$53940$ instances of 
$9$ physical attributes (both numerical and categorical features $\mathbf{x}$) and 
prices (target $\mathbf{y}$) of diamonds.

\textbf{IMDB Reviews}~\cite{ds_imdb} 
is a TFDS~\cite{tfds} 
dataset that contains $25000$ training samples 
and $25000$ testing samples
of movie reviews in a text format. 
Before passing the samples to the model $\Phi$, 
we pre-process the raw text data with spaCy~\cite{honnibal2020spacy}
\textit{en\_core\_web\_lg} pipeline which converts 
a text review 
into a $300$-dimensional vector of numbers.

The optimal sets of hyper-parameters 
are presented in Table~\ref{tab:hps_sl}.

\subsection{Learning LQR Controllers}\label{apx:lqr}

\begin{table*}[t]
\centering
\caption{Optimal Hyper-parameters for BDT and UAV}
\label{tab:hps_lqr}
\resizebox{\textwidth}{!}{%
\begin{tabular}{@{}lccccc@{}}
\toprule
\textbf{Optimizer} & \textbf{Learning Rate} & \textbf{Regularizer} & \textbf{Momentum} & \textbf{Line Search} & \textbf{\#CG Iterations} \\ \midrule
\multicolumn{6}{c}{\textbf{BDT}} \\
SGD & 0.0000005 & - & - & - & - \\
Adam & 0.1 & - & - & - & - \\
EGN & 1.0 & 1.0 & 0.0 & False & - \\
SGN & 1.0 & 1.0 & - & - & 10 \\
\multicolumn{6}{c}{\textbf{UAV}} \\
SGD & 0.00008 & - & - & - & - \\
Adam & 0.02 & - & - & - & - \\
EGN & 0.2 & 1.0 & 0.0 & False & - \\
SGN & 1.0 & 1.0 & - & - & 10 \\
\bottomrule
\end{tabular}%
}
\end{table*}

\begin{table*}[t]
\centering
\caption{Optimal Hyper-parameters for Acrobot-v1 and Freeway-v1}
\label{tab:hps_dqn}
\resizebox{\textwidth}{!}{%
\begin{tabular}{@{}lccccc@{}}
\toprule
\textbf{Optimizer} & \textbf{Learning Rate} & \textbf{Regularizer} & \textbf{Momentum} & \textbf{Line Search} & \textbf{\#CG Iterations} \\ \midrule
\multicolumn{6}{c}{\textbf{Acrobot-v1}} \\
SGD & 0.001 & - & - & - & - \\
Adam & 0.0003 & - & - & - & - \\
EGN & 0.1 & 1.0 & 0.0 & False & - \\
SGN & 0.005 & 1.0 & - & - & 3 \\
\multicolumn{6}{c}{\textbf{Freeway-v1}} \\
SGD & 0.1 & - & - & - & - \\
Adam & 0.0003 & - & - & - & - \\
EGN & 0.4 & 1.0 & 0.0 & False & - \\
SGN & 0.5 & 1.0 & - & - & 5 \\
\bottomrule
\end{tabular}%
}
\end{table*}

Here, we define the problem of learning an 
LQR controller more formally.
Given a discrete time-invariant linear 
system with continuous states $\mathcal{S}\in\rr^{n_{s}}$ 
and actions $\mathcal{A}\in\rr^{n_{a}}$
of form
$T(s,a)=\mathbf{A}s+\mathbf{B}a + e$
and a reward function 
$r(s,a)=s^{\top}\mathbf{Q}s+a^{\top}\mathbf{R}a$
our task is to learn the optimal value function $v^{*}(s)$
and the optimal policy $\pi^{*}(s)$
by interacting with $T(s,a)$,
where
$\mathbf{A}\in\rr^{n_{s}\times n_{s}}$
and
$\mathbf{B}\in\rr^{n_{a}\times n_{s}}$
are system matrices,
$e\sim\mathcal{N}\left(0,\Sigma\right)$
is Gaussian noise,
$\mathbf{Q}\in\rr^{n_{s}\times n_{s}}$ 
is a negative semi-definite state reward matrix
and $\mathbf{R}\in\rr^{n_{a}\times n_{a}}$ 
is a negative definite action reward matrix.

It is well-known~\cite{hazan2022introduction}
that the optimal value  
and policy functions have the form:
\begin{align}
v^{*}(s)	&=s^{\top}\mathbf{P}s + V_0, &
\pi^{*}(s)	&=\mathbf{K} s,
\end{align}
where $\mathbf{P} \in\rr^{n_{s}\times n_{s}}$
is a negative semi-definite 
matrix,
$\mathbf{K} \in\rr^{n_{a}\times n_{s}}$ 
is a state feedback matrix,
and $V_{0}=\gamma\left(1-\gamma\right)^{-1}\text{Tr}\left(\mathbf{P}\Sigma\right)$.

To learn the optimal controller 
from data we can utilize Generalized Policy Iteration~\cite{bradtke1994adaptive, bradtke1996linear} (Algorithm~\ref{alg:lqr_gpi}).
\begin{algorithm}[t]
	\caption{Generalized Policy Iteration for LQR}
	\label{alg:lqr_gpi}
	\begin{algorithmic}[1]
		\STATE {\bfseries Input:} initial stabilizing policy $\mathbf{K}_{0}$, 
		initial weights $\mathbf{w}_0$, learning rate $\alpha$, 
		discount factor $\gamma$, tolerance $\eta=10^{-8}$, LM regularizer $\epsilon$.
		\STATE Set policy iteration counter $p=1$
		\REPEAT{}

                \STATE Given $\mathbf{K}_{p-1}$ estimate 
                the corresponding weights $\mathbf{w}$ through a policy evaluation 
                algorithm, e.g., Algorithm~\ref{alg:lqr_pe}
      
    		\STATE Convert weights $\mathbf{w}$ 
    		to a matrix $\mathbf{M}$ 
    		\IF{$\mathbf{M}_{aa}$ is not positive-definite}
    		    \STATE \textbf{Return} Error
    		\ENDIF
    		\STATE Improve policy $\mathbf{K}_p=-\mathbf{M}_{aa}^{-1} \mathbf{M}_{as}$
    		\STATE Set $p=p+1$
		
		\UNTIL{$\left\Vert \mathbf{K}_{p}-\mathbf{K}_{p-1}\right\Vert <\eta$}
		\STATE \textbf{Return} $\mathbf{K}_{p}$
	\end{algorithmic}
\end{algorithm}

\begin{algorithm}[t]
	\caption{Policy Evaluation for LQR}
	\label{alg:lqr_pe}
	
	\begin{algorithmic}[1]
		\STATE {\bfseries Input:} policy $\mathbf{K}$, 
		initial weights $\mathbf{w}_0$, learning rate $\alpha$, 
		discount factor $\gamma$, tolerance $\eta=10^{-8}$.

            \STATE Set policy evaluation counter  $i=1$
            
            \STATE Initialize $S_1$        
            
            \REPEAT{}
        	\STATE Choose action $A_i \sim \pi(S_i)$
                by following an exploratory policy $\pi(s)=\mathbf{K}s+e$
                \STATE Execute action $A_i$ and observe $R_i, S_{i+1}$
                \STATE Obtain $A'$ by following a greedy policy $\pi(s)=\mathbf{K}s$
    
                \STATE Convert $[S_i,A_i]$ and $[S_{i+1}, A']$ to
                quadratic feature vectors $\mathbf{x}$ 
                and $\mathbf{x'}$  
                
                \STATE Calculate  $\mathbf{d}_{i}=\left(R_{i}+\mathbf{w}_{i-1}^{\top}(\gamma\mathbf{x}'-\mathbf{x})\right)\mathbf{x}$ 
    
            \STATE Update weights $\mathbf{w}_{i} \leftarrow \mathbf{w}_{i-1}+\alpha_{i} \mathbf{d}_{i}$
                
                \STATE Set $i=i+1$
    	\UNTIL{$\left\Vert \mathbf{w}_{i}-\mathbf{w}_{i-1}\right\Vert_{\infty} <\eta$}
    		
		\STATE \textbf{Return} $\mathbf{w}_i$
	\end{algorithmic}
 
\end{algorithm}

\paragraph{BDT} System matrices for the model of a binary 
distillation tower (BDT) follow~\cite{davison1990benchmark}.
The matrices represent a continuous system. 
We discretize the system using a 
Zero-Order Hold (ZOH) method with a sampling rate 
of $\Delta T=0.1s$.

\paragraph{UAV} System matrices for the linearized vertical
plane dynamics of an 
aircraft (UAV) are taken from \cite{hung1982multivariable}.
The matrices represent a continuous system. 
We discretize the system using a 
Zero-Order Hold (zoh) method with a sampling rate 
of $\Delta T=0.1s$.

The optimal sets of hyper-parameters for LQR
are presented in Table~\ref{tab:hps_lqr}.

\subsection{Reinforcement Learning with DQN}

\paragraph{Acrobot} \texttt{Acrobot-v1} is an OpenAI gym \cite{openai_gym} environment where the goal is to swing the free end of the connected joints 
above a given height in as few steps as possible. 
Transitions to any non-terminal state yield reward $R_t=-1$.

\paragraph{Freeway} \texttt{Freeway-v1}
is a part of MinAtar~\cite{young19minatar} package that 
emulates the original Atari Freeway game
which plays out on a $10 \times 10$ grid.
The goal is to reach the top of the screen
starting at the bottom of the screen maneuvering 
the obstacles appearing on the screen. A reward of $+1$
is given upon reaching the top of the screen.

The optimal sets of hyper-parameters for reinforcement learning
with DQN are presented in Table~\ref{tab:hps_dqn}.

\subsection{Limitations}\label{apx:limitations}

Table~\ref{tab:batch_times} reports per-step wall-clock times 
for different batch sizes and MLP models, 
split into direction-finding and remaining stages.

\begin{table}[t]
\centering

\caption{
Wall-clock time (ms) per step for different batch sizes and MLP models, split into direction finding stage (``Solve'') 
and remainder (``Other''). 
Means over 1000 runs (NVIDIA RTX A4000 GPU).
}
\label{tab:batch_times}
\begin{tabular}{c
  *{4}{cc}
}
\toprule
\textbf{Batch Size} 
& \multicolumn{2}{c}{\textbf{MLP 1K}} 
& \multicolumn{2}{c}{\textbf{MLP 10K}} 
& \multicolumn{2}{c}{\textbf{MLP 100K}} 
& \multicolumn{2}{c}{\textbf{MLP 1M}} \\
& Solve & Other
& Solve & Other
& Solve & Other
& Solve & Other \\
\midrule
8    & 0.135 & 0.120 & 0.177 & 0.118 & 0.176 & 0.120 & 0.492 & 0.274 \\
16   & 0.156 & 0.114 & 0.169 & 0.111 & 0.219 & 0.144 & 0.629 & 0.429 \\
32   & 0.165 & 0.112 & 0.196 & 0.113 & 0.304 & 0.183 & 0.943 & 0.725 \\
64   & 0.223 & 0.099 & 0.265 & 0.111 & 0.398 & 0.235 & 1.958 & 1.351 \\
128  & 0.326 & 0.106 & 0.381 & 0.149 & 0.734 & 0.376 & 3.489 & 2.576 \\
256  & 0.555 & 0.098 & 0.716 & 0.197 & 1.397 & 0.623 & 10.035 & 5.459 \\
512  & 1.381 & 0.140 & 1.781 & 0.258 & 4.220 & 0.861 & 38.511 & 6.239 \\
\bottomrule
\end{tabular}

\end{table}

\end{document}